\documentclass{article}

\usepackage{microtype}
\usepackage{graphicx}
\usepackage{caption}
\usepackage{subcaption}
\usepackage{booktabs}
\usepackage{listings}
\usepackage{placeins}
\usepackage{makecell}

\usepackage{hyperref}

\usepackage[accepted]{icml2023}

\usepackage{amsmath}
\usepackage{amssymb}
\usepackage{mathtools}
\usepackage{amsthm}

\usepackage[capitalize,noabbrev]{cleveref}

\usepackage{tikz}
\usetikzlibrary{calc}
\usepackage{dsfont}
\usepackage[createShortEnv]{proof-at-the-end}
\usepackage{verbatim}
\usepackage{multicol}

\definecolor{lightyellow}{rgb}{1,0.95,0.8}
\definecolor{lightgreen}{rgb}{0.83,0.91,0.83}

\theoremstyle{plain}
\newtheorem{theorem}{Theorem}[section]

\newtheorem{lemma}[theorem]{Lemma}
\newtheorem{corollary}[theorem]{Corollary}
\theoremstyle{definition}

\theoremstyle{remark}

\newtheorem{observation}[theorem]{Observation}

\def\mW{\boldsymbol{W}}

\usepackage[textsize=tiny]{todonotes}

\icmltitlerunning{\textsc{Graphtester}: Exploring Theoretical Boundaries of GNNs on Graph Datasets}

\begin{document}

\twocolumn[
\icmltitle{\textsc{Graphtester}: Exploring Theoretical Boundaries of\\GNNs on Graph Datasets}

\begin{icmlauthorlist}
\icmlauthor{M. Eren Akbiyik*}{eth}
\icmlauthor{Florian Gr\"{o}tschla*}{eth}
\icmlauthor{B\'{e}ni Egressy*}{eth}
\icmlauthor{Roger Wattenhofer}{eth}
\end{icmlauthorlist}

\icmlaffiliation{eth}{ETH Zürich, Switzerland}

\icmlcorrespondingauthor{M. Eren Akbiyik}{eakbiyik@ethz.ch}

\icmlkeywords{Machine Learning, ICML}

\vskip 0.3in
]



\printAffiliationsAndNotice{\icmlEqualContribution} 

\begin{abstract}
Graph Neural Networks (GNNs) have emerged as a powerful tool for learning from graph-structured data. 
However, even state-of-the-art architectures have limitations on what structures they can distinguish, imposing theoretical limits on what the networks can achieve on different datasets.
In this paper, we provide a new tool called \textsc{graphtester} for comprehensive analysis of the theoretical capabilities of GNNs for various datasets, tasks, and scores. 
We use \textsc{graphtester} to analyze over 40 different graph datasets, determining upper bounds on the performance of various GNNs based on the number of layers. 
Further, we show that the tool can also be used for Graph Transformers using positional node encodings, thereby expanding its scope.
Finally, we demonstrate that features generated by \textsc{Graphtester} can be used for practical applications such as Graph Transformers, and provide a synthetic dataset to benchmark node and edge features, such as positional encodings.
The package is freely available at the following URL: \href{https://github.com/meakbiyik/graphtester}{https://github.com/meakbiyik/graphtester}.
\end{abstract}

\section{Introduction}

Graph-structured data is ubiquitous in various domains, including social networks \citep{newman2003structure}, chemistry \citep{dobson2003distinguishing}, and transport \citep{barthelemy2011spatial}. Analyzing and learning from such data is critical for understanding complex systems and making informed decisions. Graph Neural Networks (GNNs) \citep{scarselli2009graph,kipf2017semi} have emerged as popular tools for learning from graph data due to their ability to capture local and global patterns \citep{bronstein2017geometric, wu2020comprehensive}.

The main approach for analyzing the theoretical power of GNN architectures relies on showing equivalence to the Weisfeiler-Lehman (WL) graph isomorphism test \citep{xu2019powerful}. Standard message-passing GNNs are bounded in power by the 1-WL test, and this can be used as a basis for calculating upper bounds on the performance of said GNNs on graph classification datasets \citep{zopf20221wl}. 
We can extend this concept to different tasks, where these upper bounds can tell us what performance is achievable on a given graph dataset with a GNN.
For example, in the task of categorical predictions for nodes, as long as we can predict the correct category for every node, we will get perfect accuracy. 
This only works if we can differentiate nodes that have to make different predictions. Once two nodes ``see'' precisely the same surrounding, they will always come to the same conclusion although they might need to make different predictions (refer Figure \ref{fig:example_graphs} for an example).
By using the equivalence to the 1-WL algorithm, we can identify structurally identical nodes for a GNN and optimize the overall accuracy by assigning the majority label amongst similar nodes.
We will thus get a theoretical upper bound that gives us insight into the solvability of a graph dataset.

We present a tool called \textsc{Graphtester} that computes these metrics in an automated way and answers the question: What is the best performance one can achieve on a given dataset with a GNN? How does this upper bound improve once we add node or edge features? We further show that \textsc{Graphtester} is not only applicable to GNNs following the message-passing formula but also graph transformers (GTs), a more recent development integrating a generalized attention mechanism, if restricted to positional encodings for nodes.

\begin{figure}[!tb]
\begin{minipage}[b]{0.48\linewidth}
    \centering
    \begin{tikzpicture}[scale=1.25, every node/.style={draw, circle}, node distance={5mm}] 
      \draw
        (0.389, -0.385) node (0)[fill=lightgreen]{}
        (0.389, 0.385) node (1)[fill=lightgreen]{}
        (-1.0, 0.0) node (2)[fill=lightyellow]{}
        (-0.389, -0.385) node (3)[fill=lightgreen]{}
        (-0.389, 0.385) node (4)[fill=lightgreen]{}
        (1.0, 0.0) node (5)[fill=lightyellow]{};
      \begin{scope}[-]
        \draw (0) to (5);
        \draw (0) to (4);
        \draw (0) to (3);
        \draw (1) to (5);
        \draw (1) to (4);
        \draw (1) to (3);
        \draw (2) to (4);
        \draw (2) to (3);
      \end{scope}
    \end{tikzpicture}
\end{minipage}
\hfill
\begin{minipage}[b]{0.48\linewidth}
  \begin{tikzpicture}[scale=1.25, every node/.style={draw, circle}, node distance={5mm}, main/.style = {draw, circle}]       
      \draw
        (0.389, -0.385) node (0)[fill=lightgreen]{}
        (0.389, 0.385) node (1)[fill=lightgreen]{}
        (-0.389, 0.385) node (2)[fill=lightgreen]{}
        (-0.389, -0.385) node (3)[fill=lightgreen]{}
        (-1.0, 0.0) node (4)[fill=lightyellow]{}
        (1.0, 0.0) node (5)[fill=lightyellow]{};
      \begin{scope}[-]
        \draw (0) to (5);
        \draw (0) to (3);
        \draw (0) to (1);
        \draw (1) to (5);
        \draw (1) to (2);
        \draw (2) to (4);
        \draw (2) to (3);
        \draw (3) to (4);
      \end{scope}
    \end{tikzpicture}
\end{minipage}
\caption{Two non-isomorphic graphs that cannot be distinguished by 1-WL. The colors stand for the stabilized 1-WL labels.}
\label{fig:example_graphs}
\vskip -0.2in
\end{figure}
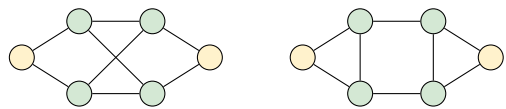

Our contributions can be summarized as follows:

\begin{itemize}
\item We present and use \textsc{Graphtester} to analyze over 40 graph datasets to determine upper bound scores for various GNNs. The tool provides support for edge features, different performance measures, different numbers of GNN layers, and higher-order WL tests. It also comes with a synthetic dataset to benchmark features for nodes and edges, such as positional encodings.

\item We prove that graph transformers making use of positional encodings for nodes are also bounded in power by the Weisfeiler-Lehman (1-WL) graph isomorphism test, thereby resulting in the same upper bounds as a GNN with additional encoding information and making our tool applicable as well. In addition, we extend the existing proofs for GNNs to cover edge features.



\end{itemize}

The rest of this paper is organized as follows: Section \ref{sec:related_work} discusses related works; Section \ref{sec:analysis} presents our theoretical analysis; Section \ref{sec:graphtester} introduces \textsc{Graphtester} package; and Section \ref{sec:conclusion} concludes the paper.



\section{Related Work}
\label{sec:related_work}

\subsection{Graph Neural Networks}

Graph Neural Networks (GNNs) have been widely studied as an effective way to model graph-structured data \citep{scarselli2009graph,kipf2017semi,bronstein2017geometric,wu2020comprehensive}. GNNs learn by propagating information through the graph and capturing local and global patterns. They have been applied to various tasks, such as node classification \citep{kipf2017semi}, link prediction \citep{schlichtkrull2018modeling}, and graph classification \citep{duvenaud2015convolutional}. There are now many variants, including Graph Convolutional Networks (GCNs) \citep{kipf2017semi}, GraphSAGE \citep{hamilton2017inductive}, and Graph Attention Networks (GATs) \citep{velivckovic2018graph}.

\subsection{Graph Transformers}

Graph transformers are a more recent development in the GNN literature, inspired by the success of the Transformer architecture in natural language processing \citep{vaswani2017attention}. They employ self-attention mechanisms to model interactions between all pairs of nodes in graphs \citep{dwivedi2020generalization, kreuzer2021rethinking, dwivedi2021graph, wu2022graphtrans}. This allows them to capture long-range dependencies with relatively few layers. However, this also comes with a high computational cost, limiting their applicability to very large graphs. Some approaches have been proposed to overcome such limitations and produce scalable architectures \citep{rampavsek2022recipe_graphGPS, wu2022nodeformer}. Graph transformers have shown promising results on various graph-based tasks, and their flexibility has led to a growing interest in understanding their capabilities and limitations. An important component when analyzing their capabilities is the choice of positional node encodings. 

\subsection{Theoretical Analysis of GNNs}

The Weisfeiler-Lehman (WL) test is a well-known graph isomorphism test that has been linked to the expressive power of message-passing-based GNNs \citep{morris2019weisfeiler, xu2018powerful}. 
There is a hierarchy of WL tests, and the 1-dimensional WL (1-WL) test has been shown to upper-bound the expressive power of many GNN architectures \citep{xu2018powerful}.
The test iteratively refines node labels based on the labels of neighboring nodes, providing a way to compare the structure of different nodes, graphs, or subgraphs. 
The theoretical connection to the WL test provides valuable insights into the representational power of GNNs, but it leaves the question of incorporating edge features largely unaddressed. In this paper, we extend this line of work by proving that in the absence of positional encodings, graph transformers are equivalent to the 2-WL algorithm, even when using edge features.
Given the importance of specific subgraph patterns for certain applications, some works assess the theoretical power of GNNs on more tangible scales \citep{chen2020canSubgraphCounting, papp2022theoretical}. 

\subsection{Positional Encodings for GNNs}

GNNs, including graph transformers, face challenges distinguishing nodes and graphs. 
For example, a basic message-passing GNN is not able to distinguish two three-cycles from a six-cycle.
One common approach to addressing this issue, especially with graph transformers, is to use positional encodings or pre-coloring methods to provide additional global information about node positions \citep{morris2019weisfeiler, maron2020provably, dwivedi2020generalization}. Several positional encoding methods have been proposed, such as eigenvector-based encodings \citep{dwivedi2020benchmarking}, Poincaré embeddings \citep{skopek2020message}, and sinusoidal encodings \citep{li2018adaptive}. These methods aim to improve the expressivity of GNNs by enriching the node features with positional information that can help GNNs better capture the graph structure. In this paper, we focus on deterministic positional encoding methods, 
as they allow us to quantify fixed and provable improvements.

\subsection{Analysis of Graph Datasets}

Closest to our work is \citet{zopf20221wl}, which analyses several standard GNN benchmark datasets for graph classification with respect to the 1-WL test. They compute upper bounds on the accuracy achievable by 1-WL GNNs and confirm that expressiveness is often not the limiting factor in achieving higher accuracy. Our work can be seen as an extension of this paper, going beyond simple graph classification to all graph-based tasks. \textsc{Graphtester} provides support for edge features, positional encodings, node/link targets, various performance measures, and higher-order WL tests. In addition, we lay the theoretical basis for extending the analysis to graph transformers, thereby further expanding the scope of \textsc{graphtester}.

\section{Theoretical Analysis}
\label{sec:analysis}

In this section, we provide the necessary theoretical analysis to use both edge features and graph transformers restricted to positional encodings for nodes in our framework.

\subsection{Preliminaries}

Let $G=(V, E, \mathbf{X}^V, \mathbf{X}^E)$ be an undirected graph, where $V$ and $E$ denote the node and edge sets, and $\mathbf{X}^V, \mathbf{X}^E$ denote the node and edge feature matrices. The feature matrices have shapes $(|V|\times d_V)$ and $(|E|\times d_E)$ respectively, with $d_V$ and $d_E$ representing the number of different labels/colors each node and edge have. $\mathcal{N}(v)$ represents the set of neighbours of node $v \in V$.

\paragraph{1-Weisfeiler-Lehman} The 1-WL algorithm, also known as naïve vertex classification or as color refinement, is one of the early attempts at the graph isomorphism (GI) problem. Variants of this algorithm are still employed in practical implementations of GI testers~\citep{Kiefer2020WL} such as \textit{nauty},  \textit{Traces}~\citep{mckay2013nauty}, and \textit{Bliss}~\cite{junttila2007bliss}. The iterative algorithm outputs a stable coloring of the nodes in a graph through a combination of neighborhood aggregation and hashing and is described in Algorithm \ref{alg:1wl}. 
The function $hash$ is an idealized perfect hash function that we assume to be injective.
We know that the induced partitioning of the nodes by color stabilizes after at most $|V| - 1$ iterations~\cite{Kiefer2020WL}, resulting in the maximum number of rounds we execute. Throughout the definitions, we use curly braces to refer to multisets.

\begin{algorithm}
\caption{1-Weisfeiler-Lehman (1-WL)}\label{alg:1wl}
\begin{algorithmic}
\REQUIRE $G=(V,E,\mathbf{X}^V)$
\STATE $c_v^{(0)} \gets hash(\mathbf{X}^V_{v}) \;\forall v \in V$ 
\FOR {$i \gets 1 \textbf{ to } (|V| - 1)$}
    \STATE $c_v^{(i)} \gets hash(c_v^{(i-1)}, \{c_w^{(i-1)} : w \in \mathcal{N}(v)\}) \quad \forall v \in V$ 
\ENDFOR
\STATE \textbf{return} $c_v^i$
\end{algorithmic}
\end{algorithm}

\paragraph{$k$-Weisfeiler-Lehman} 
The $k$-WL algorithm is a $k$-dimensional extension of Color Refinement where the algorithm hashes over subgraphs of node k-tuples instead of single nodes. The algorithm can be seen in Algorithm \ref{alg:kwl}. In addition to the standard notation, we use $G[U]$ to represent the subgraph of $G$ induced by selecting the set of nodes $U \subseteq V$. Induced subgraphs include all edges between the selected nodes in the original graph, as well as node and edge attributes associated with them. Furthermore, we use $\mathcal{N}_i(\mathbf{v})$ to represent the neighborhood of k-tuples of node $\mathbf{v}$ at the index $i$. That is, for a k-tuple $\mathbf{v}=(v_1,v_2,...,v_k)$ and $v_i \in V$, the neighborhood at index $i$ can be written as a multiset of k-tuples
\begin{align*}
    \mathcal{N}_i(\mathbf{v}) := \left\{(v_1,...,v_{i-1},w,v_{i+1},...,v_k) : w \in V \right\}.
\end{align*}

\begin{algorithm}
\caption{$k$-Weisfeiler-Lehman ($k$-WL)}\label{alg:kwl}
\begin{algorithmic}
\REQUIRE $G=(V,E,\mathbf{X}^V,\mathbf{X}^E)$
\STATE $c_\mathbf{v}^{(0)} \gets hash(G[\mathbf{v}]) \;\forall \mathbf{v} \in V^k$ 
\FOR {$i \gets 1 \textbf{ to } (|V|^k - 1)$}
    \STATE $c_{\mathbf{v},j}^{(i)} \gets \{c_{\mathbf{w}}^{(i-1)} : \mathbf{w} \in \mathcal{N}_j(\mathbf{v})\} \;\forall j \in [1..k]$
    \STATE $c_\mathbf{v}^{(i)} \gets hash(\{c_\mathbf{v}^{(i-1)}, c_{\mathbf{v},1}^{(i)},...,c_{\mathbf{v},k}^{(i)}\})$ 
\ENDFOR
\STATE \textbf{return} $c_\mathbf{v}^{(i)}$
\end{algorithmic}
\end{algorithm}

\paragraph{Equivalence of 1-WL to $k$-WL for $k=2$} It has been shown by Immerman and Lander \citep{immerman1990cr} that for graphs $G$ and $H$, 1-WL does not distinguish $G$ and $H$ if and only if $G$ and $H$ are $\mathsf{C}^2$-equivalent \citep{immerman1990cr}. Furthermore, as shown by Cai, F\"urer and Immerman~\citep{cfi1989cr}, $k$-WL does not distinguish $G$ and $H$ if and only if $G$ and $H$ are $\mathsf{C}^k$-equivalent. Consequently, 1-WL can distinguish $G$ and $H$ if and only if 2-WL can also distinguish $G$ and $H$. This insight is crucial when extending 1-WL to use edge features, as the preservation of this hierarchy is important to connect Graph Neural Networks and graph transformers to the $k$-WL literature in analyzing their expressivity.

\paragraph{Graph Neural Networks} 
GNNs comprise multiple layers that repeatedly apply neighborhood aggregation and combine functions to learn a representation vector for each node in the graph. Rigorously, for an input graph $G=(V,E,{\textbf{X}}^V)$, the \textit{i}-th layer of a GNN can be written as
\begin{align*}
    c_v^{(i)} &= \text{COMBINE}^{(i)}\left(c_v^{(i-1)},\right.\\ &\quad\quad\left.\text{AGGREGATE}^{(i)}\left(\left\{c_w^{(i-1)} : w \in \mathcal{N}(v)\right\}\right)\right),
\end{align*}
where $c_v^{(i-1)}$ represents the state of node $v$ after layer $(i-1)$.

\paragraph{Graph Transformers} Transformer models have been widely used in modeling sequence-to-sequence data in different domains \citep{vaswani2017attention}. Although the attention mechanism has commonly been used to learn on graph-structured data \citep{velivckovic2018graph}, the use of transformers is relatively recent. 
A graph transformer layer relies on global self-attention and is parameterized by query, key, and value matrices $\mW^Q, \mW^K, \mW^V \in \mathbb{R}^{d_{\text{in}} \times d_{\text{out}}}$, where $d_{\text{in}}$ is the embedding dimension of nodes before the application of the transformer layer and $d_{\text{out}}$ is the output dimension. For the sake of simplicity, we restrict ourselves to single-headed attention. We assume that node embeddings $c^{(i-1)}_v$ are stacked in a matrix $\mathbf{C}^{(i-1)} \in \mathbb{R}^{n \times d_{\text{in}}}$.
$\mathbf{C}$ is then projected with the query and key matrices before a softmax function is applied row-wise and the value matrix is multiplied:
\begin{align*}
    \textsc{Attn}(\mathbf{C}^{(i)}) = \qquad \qquad \qquad \qquad \qquad \qquad \qquad \qquad \\ \text{softmax}\left(\frac{\mathbf{C}^{(i-1)}\mW^Q (\mathbf{C}^{(i-1)} \mW^K)^T}{\sqrt{d_{\text{out}}}}\right) \ \mathbf{C}^{(i-1)} \mW^{V}
\end{align*}
States $\mathbf{C}^{(i-1)}$ can be passed through a learnable function before and after the global attention \textsc{Attn} is applied.
Positional encodings are commonly used with graph transformers to give every node a sense of where it is located in the graph. Positional encodings can come in the form of node encodings \citep{rampavsek2022recipe_graphGPS} that are essentially features added to the nodes before the attention block is applied or node-pair encodings, where each node-pair is endowed with features such as shortest-path distances~\citep{ying2021transformers}. 
Node-pair encodings have the downside that the full attention matrix has to be materialized. In this case, one cannot profit from faster attention mechanisms~\citep{rampavsek2022recipe_graphGPS} that scale better than $\mathcal{O}(n^2)$, making it practically infeasible.
Here, we restrict ourselves to node encodings.

\begin{theoremE}
    The 1-WL test is at least as powerful as GTs with positional encodings for nodes, e.g., GraphTrans~\citep{wu2022graphtrans} or GraphGPS~\citep{rampavsek2022recipe_graphGPS}, if node encodings are also provided as initial color classes for the 1-WL algorithm.
\end{theoremE}
\textit{Proof sketch.}
To prove that 1-WL is an upper bound in terms of expressiveness for GTs with node encodings, we first consider the color classes of a 1-WL execution on the fully-connected graph, instead of the original topology. As we input positional encodings as color classes to 1-WL, and we can reconstruct all attention scores from the 1-WL labels in every iteration, it becomes clear that 1-WL can simulate a GT. We then show that any two nodes with the same color class in a fully connected graph will stay in the same color class, meaning no more refinement of color classes is possible for a transformer layer.

\begin{proofE}
    First, we consider the fully-connected graph for 1-WL and show that this is at least as powerful as a GT with only node features. 
    As 1-WL hashes all neighbor states with an injective function, we can observe states from all nodes in the graph in the aggregated multiset at node $v$. We can then compute the outcome of the attention module at every node:
    \begin{align*}
        \textsc{Attn}(c^{(i)}_v) = \text{softmax}\left(\frac{c^{(i-1)^T}_v\mW^Q (\mathbf{C}^{(i-1)^T} \mW^K)^T}{\sqrt{d_{\text{out}}}}\right) \ \mathbf{C}^{(i-1)} \mW^{V},
    \end{align*}
    where $C^{(i-1)}$ is computed by stacking all aggregated states and the state $c^{(i)}_v$ as row-vectors in a matrix. The order does not matter here. This only works because we receive the state from every node, which is not the case if the graph is not fully connected.
    We further show that color classes for 1-WL in a fully-connected graph are refined no further, meaning that the execution of a global self-attention layer does not improve expressiveness.
    Consider two nodes $u$ and $v$ with the same color $c_u^{(i-1)} = c_v^{(i-1)}$in such a setting. Then, both nodes will receive exactly the same multiset of neighboring node states: It contains all $c_w^{(i-1)}$ for $w \neq u, v$, plus $c_u^{(i-1)}$ for node $v$ and $c_v^{(i-1)}$ for node $v$, which are the same. 

    The proof shows that one layer of global self-attention does not improve expressiveness, as color classes cannot be refined, which makes it applicable to GraphTrans where only one such layer is applied at the end, but also GraphGPS, where layers of message-passing are interleaved with global self-attention layers. 
    It should be noted that this does not work for architectures adding information to any node-pair for the attention computation, such as shortest path distances in Graphormer~\cite{ying2021transformers}.

\end{proofE}

\subsection{Edge-Feature-Aware 1-WL Algorithm}

We now present the edge-feature-aware 1-WL (1-WLE). At each iteration, the algorithm updates the node labels based on both the neighboring node labels and the edge labels of the connecting edges. Formally, the edge-feature-aware 1-WL is defined in Algorithm \ref{alg:1wle}.

\begin{algorithm}
\caption{Edge-Feature-Aware 1-WL (1-WLE)}\label{alg:1wle}
\begin{algorithmic}
\REQUIRE $G=(V,E,\mathbf{X}^V,\mathbf{X}^E)$
\STATE $c_v^{(0)} \gets hash(\mathbf{X}^V_{v}) \;\forall v \in V$ 
\FOR {$i \gets 1 \textbf{ to } (|V| - 1)$}
    \STATE $c_v^{(i)} \gets hash(c_v^{(i-1)}, \{(\mathbf{X}^E_{(v,w)}, c_w^{(i-1)}) : w \in \mathcal{N}(v)\})$ 
\ENDFOR
\STATE \textbf{return} $c_v^i$
\end{algorithmic}
\end{algorithm}

\begin{theoremE}
    The Edge-Feature-Aware 1-WL test is equivalent in power to a GIN with edge features, as proposed by Hu et al.~\citep{hu2019gineconv}.
\end{theoremE}
\textit{Proof sketch.} 
To show the equivalence between 1-WLE and GIN with edge features, one can extend the original proof for equivalence of 1-WL and GIN. What changes is that the aggregation now gets node states with additional edge labels, but an injective aggregation will still maintain this information.

\begin{proofE}
    A more rigorous proof of equivalence can be crafted by following the work on Graph Isomorphism Networks~\citep{xu2019powerful}. We will simply sketch the connection here. GIN convolution layers with edge features (also known as GINEConv) are defined as follows:
    \begin{align*}
        c_v^{(i)} &= \text{COMBINE}^{(i)}\left(c_v^{(i-1)}, \text{AGGREGATE}^{(i)}\left(\left\{\left(\mathbf{X}^E_{(v,w)}, c_w^{(i-1)}\right) \mid w \in \mathcal{N}(v)\right\}\right)\right).
    \end{align*}
    For an injective function $\text{AGGREGATE}^{(i)}$ that operates on multisets, and an injective function $\text{COMBINE}^{(i)}$, it is shown that without edge features GINs are as powerful as 1-WL~\citep{xu2019powerful}. 

    To extend this argument to GIN with edge features, we need $\text{AGGREGATE}^{(i)}$ to be an injective function operating on multisets of tuples. In GINEConv, concatenation is used to connect each neighboring node embedding $\mathbf{X}^E_{(v,w)}$ with its corresponding edge features $c_w^{(i-1)}$, whence the aggregation step continues as in GIN with an injective function acting on multisets. Since the node and edge features have fixed lengths, concatenation is injective. Then, since the composition of injective functions is also injective, this approach ensures that $\text{AGGREGATE}^{(i)}$ remains injective when edge features are used.

    The rest of the argument follows along the same lines as the proofs in \citet{xu2019powerful}. This shows that GINEConv can be as powerful as 1-WLE.

    On the other hand, to show that GINEConv (and extensions of other message-passing GNNs) can be at most as powerful as 1-WLE, one can follow the same inductive proof as in Lemma 2 of \citet{xu2019powerful}.
\end{proofE}

\subsection{Equivalence to 2-WL Test}

We now provide a proof that our edge-feature-aware 1-WL extension is equivalent to the 2-WL test. First, for some simple operator, we show that 1-WLE over a graph with edge features is equivalent to 1-WL with the same graph when the operator is applied. Then, we point to the equivalence of 2-WL with 1-WL under the operator, and finally, show the equivalence of 2-WL over a graph with edge features, and the same graph under the given operator.

\paragraph{Incidence graph operator} Consider the graphs in the form $G=(V, E, \mathbf{X}^V, \mathbf{X}^E)$. We denote operator $\mathcal{T}: G \rightarrow G$ as the incidence graph operator as follows: for the given input graph $G$ with edge labels, $\mathcal{T}$ creates a new node $w$ for each edge $(u,v) \in E$ with the edge feature assigned as the node label, connects two ends of the edge to the node with new edges $(u,w), (w,v)$, and finally removes the original edge $(u,v)$. The final graph is also referred to as the "incidence graph" of graph G. In the output graph, there are no edge labels. For an example application, see Figure \ref{fig:operator}.

\begin{figure}[!tbh]
\begin{minipage}[b]{0.48\linewidth}
    \centering
    \begin{tikzpicture}[every node/.style={scale=0.75}, node distance={20mm}, thick, main/.style = {draw, circle}] 
    \node[main] (1) {$x_1$}; 
    \node[main] (2) [above right of=1] {$x_2$}; 
    \node[main] (3) [below right of=1] {$x_3$}; 
    \node[main] (4) [above right of=3] {$x_4$}; 
    \node[main] (5) [above right of=4] {$x_5$}; 
    \node[main] (6) [below right of=4] {$x_6$}; 
    \draw (1) -- node[midway, fill=white] {a} (2); 
    \draw (1) -- node[midway, fill=white] {b} (3); 
    \draw (2) -- node[midway, fill=white] {b} (4); 
    \draw (3) -- node[midway, fill=white] {a} (4); 
    \draw (5) -- node[midway, fill=white] {c} (4); 
    \draw (6) -- node[midway, fill=white] {b} (4); 
    \end{tikzpicture}
\end{minipage}
\hfill
\begin{minipage}[b]{0.48\linewidth}
    \centering
    \begin{tikzpicture}[every node/.style={scale=0.75}, node distance={20mm}, thick, main/.style = {draw, circle}] 
    \node[main] (1) {$x_1$}; 
    \node[main] (2) [above right of=1] {$x_2$}; 
    \node[main] (3) [below right of=1] {$x_3$}; 
    \node[main] (4) [above right of=3] {$x_4$}; 
    \node[main] (5) [above right of=4] {$x_5$}; 
    \node[main] (6) [below right of=4] {$x_6$}; 
    \node[main] (7) at ($(1)!0.5!(2)$) {$a$}; 
    \node[main] (8) at ($(1)!0.5!(3)$)  {$b$}; 
    \node[main] (9) at ($(2)!0.5!(4)$)  {$b$}; 
    \node[main] (10) at ($(3)!0.5!(4)$)  {$a$}; 
    \node[main] (11) at ($(5)!0.5!(4)$)  {$c$}; 
    \node[main] (12) at ($(6)!0.5!(4)$)  {$b$}; 
    \draw (1) -- (7); 
    \draw (7) -- (2); 
    \draw (1) -- (8); 
    \draw (8) -- (3); 
    \draw (2) -- (9); 
    \draw (9) -- (4); 
    \draw (3) -- (10); 
    \draw (10) -- (4); 
    \draw (5) -- (11); 
    \draw (11) -- (4); 
    \draw (6) -- (12); 
    \draw (12) -- (4); 
    \end{tikzpicture}
\end{minipage}
\caption{Incidence graph operator $\mathcal{T}$ applied on the input graph on the right, outputs the final graph on the left with no edge features, and each edge converted to a node with the original edge label.}
\label{fig:operator}
\end{figure}
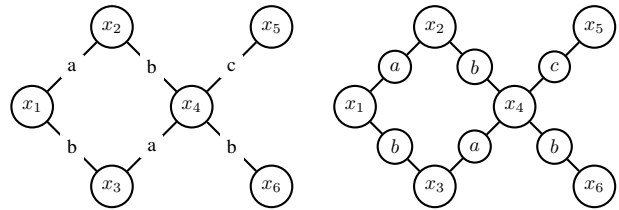

\begin{theoremE}[Equivalence of 1-WLE to 1-WL($\mathcal{T}$(G))] \label{thm:cre_incidence_cr_equivalence} 1-WLE on graph G is equivalent to 1-WL over graph $\mathcal{T}$(G) so that there is an injective map $f$ for which f(1-WLE(G)) = 1-WL($\mathcal{T}$(G)) up to isomorphism.
\end{theoremE}
\begin{proofE}
We denote the initial node labels for the original nodes $v$ in both the 1-WL and the 1-WLE algorithms as
\begin{equation*}
    c_v^{(0)} \gets hash(\mathbf{X}^V_{v}) \;\forall v \in V.
\end{equation*}
We denote the initial node labels for the newly added nodes for the 1-WL algorithm as
\begin{equation*}
    c_{(v,w)}^{(0)} \gets hash(\mathbf{X}^E_{(v,w)}) \;\forall (v,w) \in E
\end{equation*}

For the first iteration of 1-WLE over graph G, we have the node features
\begin{align*}
    c_v^{(1)} \gets hash(c_v^{(0)}, \{(\mathbf{X}^E_{(v,w)}, c_w^{(0)}) \mid w \in \mathcal{N}(v)\})
\end{align*}

For 1-WL over graph $\mathcal{T}(G)$, we distinguish the sets of original nodes $v$ and newly added nodes. For the first iteration, these sets can be written as
\begin{align*}
    c'{}_v^{(1)} &\gets hash(c_v^{(0)}, \{c_{(v,w)}^{(0)} \mid w \in \mathcal{N}(v)\})\\
    c'{}_{(v,w)}^{(1)} &\gets hash(c_{(v,w)}^{(0)}, \{c_v^{(0)}, c_w^{(0)}\})
\end{align*}

While writing the second step of 1-WL over $\mathcal{T}(G)$, WLOG we can replace the hash function with an identity function for the proof, as it is assumed to be injective:
\begin{align*}
    c'{}_v^{(2)} & \gets (c'{}_v^{(1)}, \{ c'{}_{(v,w)}^{(1)} \mid w \in \mathcal{N}(v) \}) \\
    c'{}_{(v,w)}^{(2)} &\gets (c_{(v,w)}^{(0)}, \{c_v^{(0)}, c_w^{(0)}\}, \{c'{}_v^{(1)}, c'{}_w^{(1)}\})
\end{align*}

The labels of the original nodes can be rewritten
\begin{align*}
    c'{}_v^{(2)} & \gets ( c'{}_v^{(1)}, \{ c'{}_{(v,w)}^{(1)} \mid w \in \mathcal{N}(v) \} ) \\
    & \equiv ((c_v^{(0)}, \{c_{(v,w)}^{(0)} \mid w \in \mathcal{N}(v)\}), \{c'{}_{(v,w)}^{(1)} \mid w \in \mathcal{N}(v)\} ) \\
    & \equiv ( (c_v^{(0)}, \{\mathbf{X}^{E}_{(v,w)} \mid w \in \mathcal{N}(v)\}), \{c'{}_{(v,w)}^{(1)} \mid w \in \mathcal{N}(v)\} ) \\
    & \equiv ( (c_v^{(0)}, \{\mathbf{X}^{E}_{(v,w)} \mid w \in \mathcal{N}(v)\}), \{(\mathbf{X}^{E}_{(v,w)}, \{c_v^{(0)}, c_w^{(0)}\} ) \mid w \in \mathcal{N}(v)\} ) \\
    & \equiv (c_v^{(0)}, \{(\mathbf{X}^{E}_{(v,w)}, \{c_v^{(0)}, c_w^{(0)}\}) \mid w \in \mathcal{N}(v)\})\\
    & \equiv (c_v^{(0)}, \{(\mathbf{X}^{E}_{(v,w)}, c_w^{(0)}) \mid w \in \mathcal{N}(v)\}) \\
    & \equiv c_v^{(1)}
\end{align*}

We can observe that $c{}_v^{(1)} \equiv c'{}_v^{(2)}$, that is, two steps of 1-WL over graph $\mathcal{T}(G)$ outputs equivalent labels for the original nodes as one step of 1-WLE over graph G. Now, we only need to answer whether the labels of the newly added nodes add any information to the process. We can similarly rewrite the node labels of the new nodes
\begin{align*}
    c'{}_{(v,w)}^{(2)} &\gets (c_{(v,w)}^{(0)}, \{c_v^{(0)}, c_w^{(0)}\}, \{c'{}_v^{(1)}, c'{}_w^{(1)}\}) \\
    &\equiv ( c_{(v,w)}^{(0)}, \{ c'{}_v^{(1)}, c'{}_w^{(1)} \} ) \\
    &\equiv (\mathbf{X}^E_{(v,w)}, \{ (c_v^{(0)}, \{c_{(v,k)}^{(0)} \mid k \in \mathcal{N}(v)\}) , (c_w^{(0)}, \{c_{(w,j)}^{(0)} \mid j \in \mathcal{N}(w)\}) \} )\\
    &\equiv (\mathbf{X}^E_{(v,w)}, \{ (c_v^{(0)}, \{\mathbf{X}^E_{(v,k)} \mid k \in \mathcal{N}(v)\}) , (c_w^{(0)}, \{\mathbf{X}^E_{(w,j)} \mid j \in \mathcal{N}(w)\}) \} )\\
    &\equiv ( \{ c'{}_v^{(1)}, c'{}_w^{(1)} \} ) 
\end{align*}
since $\mathbf{X}^E_{(v,w)}$ is included in both hashes $c'{}_v^{(1)}$ and $c'{}_w^{(1)}$. Indeed, the same argument shows that for all $k \geq 2$, $c'{}_{(v,w)}^{(k)} \equiv \{ c'{}_v^{(k-1)}, c'{}_w^{(k-1)} \}$, since $c^{(k)}$ is always a refinement of $c^{(k-1)}$.

We now claim that every two iterations of 1-WL over graph $\mathcal{T}(G)$ lead to the same node refinement as a single iteration of 1-WLE over $G$. That is:
\begin{equation*}
    c'{}_v^{(2k)} \equiv c_v^{(k)} \tag{for all $k \geq 1$} 
\end{equation*}
We can prove this claim by induction, following the same argument as above.
\begin{align*}
    c'{}_v^{(2(k+1))} & \gets ( c'{}_v^{(2k+1)}, \{ c'{}_{(v,w)}^{(2k+1)} \mid w \in \mathcal{N}(v) \} ) \\
    & \equiv ((c'{}_v^{(2k)}, \{c'{}_{(v,w)}^{(2k)} \mid w \in \mathcal{N}(v)\}), \{c'{}_{(v,w)}^{(2k+1)} \mid w \in \mathcal{N}(v)\} ) \\
    & \equiv (c'{}_v^{(2k)}, \{c'{}_{(v,w)}^{(2k+1)} \mid w \in \mathcal{N}(v)\} ) \\
    & \equiv (c_v^{(k)}, \{c'{}_{(v,w)}^{(2k+1)} \mid w \in \mathcal{N}(v)\} ) \\
    & \equiv (c_v^{(k)}, \{ \{c'{}_v^{(2k)}, c'{}_w^{(2k)}\} \mid w \in \mathcal{N}(v)\})\\
    & \equiv (c_v^{(k)}, \{ \{c_v^{(k)}, c_w^{(k)}\} \mid w \in \mathcal{N}(v)\})\\
    & \equiv (c_v^{(k)}, \{ c_w^{(k)} \mid w \in \mathcal{N}(v)\})\\
    & \equiv (c_v^{(k)}, \{ (\mathbf{X}^E_{(v,w)}, c_w^{(k)}) \mid w \in \mathcal{N}(v)\})\\
    & \equiv c_v^{(k+1)}
\end{align*}
The penultimate equivalence follows since $\mathbf{X}^E_{(v,w)}$ is already considered in the k\textsuperscript{th} refinement of the nodes, i.e. $c_v^{(k)}$ and $c_w^{(k)}$, for all $k \geq 1$. The final equivalence follows by definition.

Therefore, we get the same partition of nodes for 1-WL over graph $\mathcal{T}(G)$ as for 1-WLE over $G$. Hence, the two are equivalent.

In essence, the newly added (edge) nodes add no information to the hash of the overall graph after the first iteration and act merely as messengers between nodes corresponding to neighboring nodes in the original graph $G$.
\end{proofE}

\begin{observation}
    As 1-WL(G) is equivalent to 2-WL(G) as shown by Immerman et al., 1-WL($\mathcal{T}$(G)) is equivalent to 2-WL($\mathcal{T}$(G)), and consequently, 1-WLE(G) $\equiv$ 2-WL($\mathcal{T}$(G)).
\end{observation}

\begin{theoremE}[Equivalence of 2-WL($\mathcal{T}$(G)) to 2-WL(G)]
\label{thm:2wl_equivalence}
    2-WL on graph G is equivalent to 2-WL over graph $\mathcal{T}$(G) so that there is an injective map $f$ for which f(2-WL(G)) = 2-WL($\mathcal{T}$(G)) up to isomorphism.
\end{theoremE}
\begin{proofE}
We provide a high-level proof sketch for this theorem. We start with some formulas, but then aim to give the intuition behind the proof, rather than providing pages of formulas.

We separate the node tuples into tuples formed only from original nodes, only from the newly added (edge) nodes -- added by the $\mathcal{T}$ operator, and tuples formed from a mix of the two.
\begin{align*}
    c_\mathbf{v}^{(0)} &\gets hash(G[\mathbf{v}]) \;\forall \mathbf{v} \in V^2\\
    c'{}_\mathbf{v}^{(0)} &\gets hash(\mathcal{T}(G)[\mathbf{v}]) \;\forall \mathbf{v} \in V^2\\
    c'{}_\mathbf{e}^{(0)} &\gets hash(\mathcal{T}(G)[\mathbf{e}]) \;\forall \mathbf{e} \in E^2\\
    c'{}_{(e,v)}^{(0)} &\gets hash(\mathcal{T}(G)[\{e,v\}]) \;\forall \{e,v\} \in E \times V
\end{align*}

As these subgraph hashes are only determined by the features of the nodes in the tuples and the existence/features of the edges between them, we can rewrite these initializations.
\begin{align*}
    c_\mathbf{v}^{(0)} &\gets hash( \{ \mathbf{X}^V_{\mathbf{v}_1}, \mathbf{X}^V_{\mathbf{v}_2} \}, \mathds{1}_{(\mathbf{v}_1, \mathbf{v}_2) \in E}, \mathbf{X}^E_{(\mathbf{v}_1, \mathbf{v}_2)}) \;\forall \mathbf{v} \in V^2\\
    c'{}_\mathbf{v}^{(0)} &\gets hash( \{ \mathbf{X}^V_{\mathbf{v}_1}, \mathbf{X}^V_{\mathbf{v}_2} \} ) \;\forall \mathbf{v} \in V^2 \tag{There are no edges between original nodes under operator $\mathcal{T}$}\\
    c'{}_\mathbf{e}^{(0)} &\gets hash( \{ \mathbf{X}^E_{\mathbf{e}_1}, \mathbf{X}^E_{\mathbf{e}_2} \} ) \;\forall \mathbf{e} \in E^2\tag{There are no edges between newly added nodes under operator $\mathcal{T}$}\\
    c'{}_{(e,v)}^{(0)} &\gets hash( \{ \mathbf{X}^E_{e}, \mathbf{X}^V_{v} \}, \mathds{1}_{(e,v) \in E_{\mathcal{T}(G)}}) \;\forall \{e,v\} \in E \times V
\end{align*}

As in the proof of Theorem \ref{thm:cre_incidence_cr_equivalence}, we can again WLOG replace the hash function with an identity function.
For the first iteration of 2-WL over graph G, we have the 2-tuple features:
\begin{align*}
    c_\mathbf{v}^{(1)} 
    &\gets (c_\mathbf{v}^{(0)}, \{c_\textbf{w}^{(0)} \mid \textbf{w} \in \mathcal{N}_1(\mathbf{v})\} \uplus \{c_\textbf{w}^{(0)} \mid \textbf{w} \in \mathcal{N}_2(\mathbf{v})\}\})\\
    &\equiv ( c_\mathbf{v}^{(0)}, \{c_{(u, \textbf{v}_2)}^{(0)} \mid u \in V\} \uplus \{c_{(\textbf{v}_1, u)}^{(0)} \mid u \in V\})\\
    &\equiv ( ( \{ \mathbf{X}^V_{\mathbf{v}_1}, \mathbf{X}^V_{\mathbf{v}_2} \}, \mathds{1}_{(\mathbf{v}_1, \mathbf{v}_2) \in E}, \mathbf{X}^E_{(\mathbf{v}_1, \mathbf{v}_2)}),\\
    &\quad\quad\{( \{ \mathbf{X}^V_{u}, \mathbf{X}^V_{\mathbf{v}_2} \}, \mathds{1}_{(u, \mathbf{v}_2) \in E}, \mathbf{X}^E_{(u, \mathbf{v}_2)}) \mid u \in V\} \uplus \\
    &\quad\quad\{( \{ \mathbf{X}^V_{\mathbf{v}_1}, \mathbf{X}^V_{u} \}, \mathds{1}_{(\mathbf{v}_1, u) \in E}, \mathbf{X}^E_{(\mathbf{v}_1, u)}) \mid u \in V\}\})
\end{align*}

For 2-WL over graph $\mathcal{T}(G)$ in the first iteration, we have the 2-tuple features
\begin{align*}
    c'{}_\mathbf{v}^{(1)} &\gets ( c'{}_\mathbf{v}^{(0)}, \{c'{}_w^{(0)} \mid w \in \mathcal{N}_1(\mathbf{v})\} \uplus \{c_w^{(0)} \mid w \in \mathcal{N}_2(\mathbf{v})\} )\\
    &\equiv ( c'{}_\mathbf{v}^{(0)}, \{c'{}_{(u, \textbf{v}_2)}^{(0)} \mid u \in V \cup E\} \uplus \{c'{}_{(\textbf{v}_1, u)}^{(0)} \mid u \in V \cup E\}\} ).
\end{align*}

There are three cases to consider. (i) For $\textbf{v}:=(\textbf{v}_1,\textbf{v}_2) \in V^2$,  we can simplify the above expression:
\begin{align*}
    c'{}_\mathbf{v}^{(1)} &\gets ( \{ \mathbf{X}^V_{\mathbf{v}_1}, \mathbf{X}^V_{\mathbf{v}_2} \} , \{c'{}_{(u, \textbf{v}_2)}^{(0)} \mid u \in V \cup E\} \uplus \{c'{}_{(\textbf{v}_1, u)}^{(0)} \mid u \in V \cup E\} )\\
    &\equiv ( \{ \mathbf{X}^V_{\mathbf{v}_1}, \mathbf{X}^V_{\mathbf{v}_2} \}, \{ \{ \mathbf{X}^V_{u}, \mathbf{X}^V_{\mathbf{v}_2} \} \mid u \in V\} \uplus \{ \{ \mathbf{X}^V_{\mathbf{v}_1}, \mathbf{X}^V_{u} \} \mid u \in V\} \uplus \\
    &\quad\quad\{( \{ \mathbf{X}^E_{e}, \mathbf{X}^V_{\mathbf{v}_1} \}, \mathds{1}_{(e,\mathbf{v}_1) \in E_{\mathcal{T}(G)}}) \mid e \in E\} \uplus \{( \{ \mathbf{X}^E_{e}, \mathbf{X}^V_{\mathbf{v}_2} \} ,\mathds{1}_{(e,\mathbf{v}_2) \in E_{\mathcal{T}(G)}}) \mid e \in E\}\})\\
    &\equiv ( \{ \mathbf{X}^V_{\mathbf{v}_1}, \mathbf{X}^V_{\mathbf{v}_2} \} , \\
    &\quad\quad\{( \{ \mathbf{X}^E_{e}, \mathbf{X}^V_{\mathbf{v}_1} \},\mathds{1}_{(e,\mathbf{v}_1) \in E_{\mathcal{T}(G)}}) \mid e \in E\} \uplus \{( \{ \mathbf{X}^E_{e}, \mathbf{X}^V_{\mathbf{v}_2} \},\mathds{1}_{(e,\mathbf{v}_2) \in E_{\mathcal{T}(G)}}) \mid e \in E\}\})\\
\end{align*}
The final equivalence follows because all node tuples $\textbf{v}:=(\textbf{v}_1,\textbf{v}_2) \in V^2$ will receive the sets $\{ \{ \mathbf{X}^V_{u}, \mathbf{X}^V_{\mathbf{v}_2} \} \mid u \in V\}$ and $\{ \{ \mathbf{X}^V_{\mathbf{v}_1}, \mathbf{X}^V_{u} \} \mid u \in V\}$, with the only difference being the $\mathbf{X}^V_{\mathbf{v}_1}$ and the $\mathbf{X}^V_{\mathbf{v}_2}$. There this information is equivalent to knowing $\{ \mathbf{X}^V_{\mathbf{v}_1}, \mathbf{X}^V_{\mathbf{v}_2} \}$.

What does this last line represent? From the first element, we have the set of features of the two nodes in $\textbf{v}$. The last two elements tell us how many edges (with associated edge features) the nodes of the given tuple $\textbf{v}$ have in the original graph. However, it doesn't tell us whether these edges in the original graph connect the two nodes in the tuple with each other, nor whether they connect to a common neighbor thereby forming a path of length two in the original graph. So similar to 1-WL, the first iteration essentially just tells us the degrees of the two nodes, and the set of features associated with those edges.

The edge tuples gain even less from the first iteration, since the ``edge nodes'' all have degree two. However, vertex-edge tuples will ``know'' whether the edge is incident to the vertex in the original graph, thanks to the indicator function in $c'{}_{(e,v)}^{(0)}$. And then after one iteration for neighboring nodes $u$ and $v$ (with $e=(u,v) \in E$), $c'{}_{(e,v)}^{(1)}$ will ``see'' the edge $(e,u)$ and $c'{}_{(e,u)}^{(1)}$ will ``see'' the edge $(e,v)$. 

As such, when we update $c'{}_\mathbf{v}^{(2)}$ in the next iteration, then $\mathbf{v}_1$ and $\mathbf{v}_2$ will ``know'' that they are connected to a node with the corresponding features by an edge with the given edge features. 
They will also ``know'' the edge and node features of any other neighbors they might have. But this is exactly the information contained in $c_\mathbf{v}^{(1)}$. In other words two iterations of 2-WL on $\mathcal{T}(G)$ is equivalent to one iteration of 2-WL on $G$, in terms of distinguishing nodes. 

By showing that vertex-edge and edge-edge tuples do not contain any additional information and then turning this into an inductive argument, we can confirm that $c'{}_\mathbf{v}^{(2k)} \equiv c{}_\mathbf{v}^{(k)}$ for all $\mathbf{v} \in V^2$ and for all $k \geq 0$. This is similar to the proof of Theorem \ref{thm:cre_incidence_cr_equivalence}.


\end{proofE}


\begin{corollary}
    1-WLE($G$) $\equiv$ 2-WL($G$)
\end{corollary}

\begin{proofE}
    1-WLE($G$) $\equiv$ 1-WL($\mathcal{T}(G)$) by Theorem \ref{thm:cre_incidence_cr_equivalence}, 1-WL($\mathcal{T}(G)$) $\equiv$ 2-WL($\mathcal{T}(G)$ by the equivalence of 1-WL and 2-WL, and 2-WL($\mathcal{T}(G)$) $\equiv$ 2-WL($G$) by Thereom \ref{thm:2wl_equivalence}. 
\end{proofE}

We see value in pointing out here that using edge features in concatenation phase of every iteration, as shown in Algorithm \ref{alg:1wle}, is not identical with injecting edge features into the node features as a preprocessing step, which is used to overcome the deficiencies of some GNN architectures that do not natively admit edge features. The proof by counter-example can be found in Theorem \ref{thm:reduction_insufficiency}.

\begin{theoremE} \label{thm:reduction_insufficiency}
    1-WL over a graph with edge feature reduction (where edge features are concatenated with the neighboring node features before running the algorithm) is strictly weaker than 1-WLE.
\end{theoremE}
\begin{proofE}
    First, we start by focusing on the edge feature tuples, that is, a simple neighborhood aggregation of the edges.
    \begin{align*}
        x_v \gets \{X^E_{(u,v)}\: \forall u \in \mathcal{N}(v)\}
    \end{align*}
    For this simplified reduction, we provide the following counterexample in Figure \ref{fig:counterexample}. In this graph, assume all node features are identical. For this structure, 1-WLE and 1-WL with reduced edges (according to the simplified definition above) output different node hashes: the former can distinguish $x_1$ and $x_1'$ while the latter cannot. This is because node labels are symmetric for both sides of the graph when using initial features with 1-WL. With 1-WLE, they will differ after two iterations as $x_1$ is connected over an edge with label $2$ to $x_4$, whereas $x'_1$ is connected to $x'_4$ over an edge with label 1.
    \begin{figure}[!tbh]
    \centering
    \begin{tikzpicture}[every node/.style={scale=0.75}, node distance={20mm}, thick, main/.style = {draw, circle}] 
    \node[main] (1) {$x_1$}; 
    \node[main] (2) [above left of=1] {$x_2$}; 
    \node[main] (3) [below left of=2] {$x_3$}; 
    \node[main] (4) [below left of=1] {$x_4$}; 
    \node[main] (5) [below of=4] {$x_5$}; 
    \node[main] (6) [right of=1] {$x'_1$}; 
    \node[main] (7) [above right of=6] {$x'_2$}; 
    \node[main] (8) [below right of=7] {$x'_3$}; 
    \node[main] (9) [below right of=6] {$x'_4$}; 
    \node[main] (10) [below of=9] {$x'_5$}; 
    \draw (1) -- node[midway, fill=white] {0} (6); 
    \draw (1) -- node[midway, fill=white] {1} (2); 
    \draw (1) -- node[midway, fill=white] {2} (4); 
    \draw (2) -- node[midway, fill=white] {2} (3); 
    \draw (3) -- node[midway, fill=white] {1} (4); 
    \draw (4) -- node[midway, fill=white] {3} (5); 
    \draw (6) -- node[midway, fill=white] {2} (7); 
    \draw (7) -- node[midway, fill=white] {1} (8); 
    \draw (6) -- node[midway, fill=white] {1} (9); 
    \draw (8) -- node[midway, fill=white] {2} (9); 
    \draw (9) -- node[midway, fill=white] {3} (10); 
    \end{tikzpicture}
    \caption{Counterexample for the expressive power of a simplified reduction. Assume that all nodes are unlabeled, and edge labels are the numbers written over them.}
    \label{fig:counterexample}
    \end{figure}
\end{proofE}

\section{\textsc{Graphtester}: Theoretical Analysis of Graph Datasets}
\label{sec:graphtester}

\begin{figure*}
    \centering
    \includegraphics[width=\textwidth]{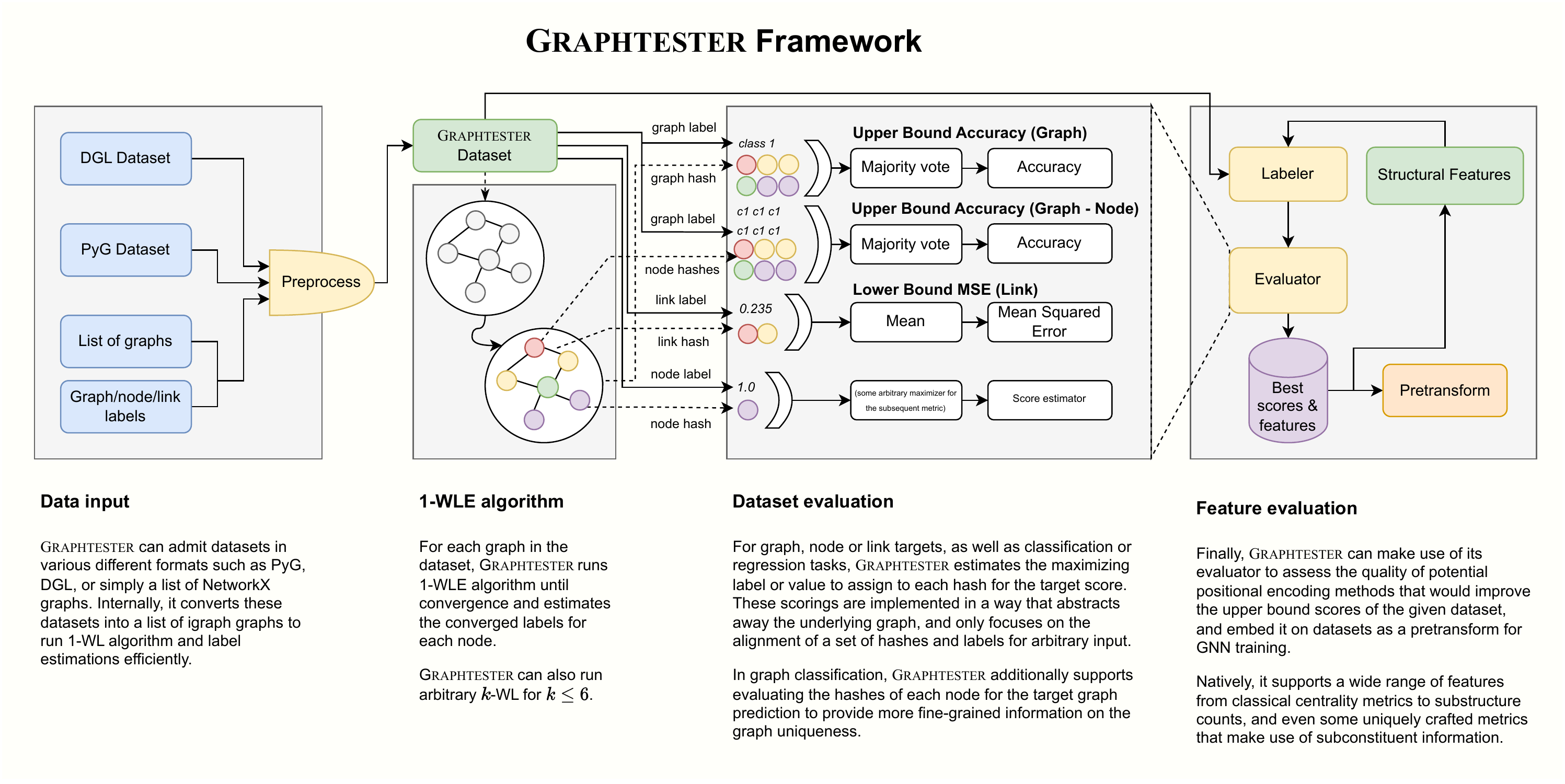}
    \caption{An overview of the \textsc{Graphtester} framework. Overall, the package has four major components: preprocessing, 1-WLE algorithm, dataset evaluation and feature evaluation.}
    \label{fig:gt}
\vskip -0.2in
\end{figure*}

\begin{table*}[!thb]
    \centering
    \caption{Analysis of over 40 datasets, providing upper bounds on GNN performance when using different numbers of GNN layers and different input features (when available). Layer count 0 is not included for the cases with edge features, since they are equivalent to the cases where they are not present.}
    \label{tab:dataset_analysis}
    \resizebox{\textwidth}{!}{%
    \begin{tabular}{l l l r r r r r r r r r r r r r r}
    \toprule
    Dataset name & Task & Metric  & \multicolumn{4}{c}{w/out features} & \multicolumn{4}{c}{w/node features} & \multicolumn{3}{c}{w/edge features} & \multicolumn{3}{c}{w/both features} \\
    \cmidrule(lr){4-7} \cmidrule(lr){8-11} \cmidrule(lr){12-14} \cmidrule(lr){15-17}
    & & & 0 & 1 & 2 & 3 & 0 & 1 & 2 & 3 & 1 & 2 & 3 & 1 & 2 & 3 \\
    \midrule
    AIDS~\citep{morris2020tudatasets} & Graph Cl. & Accuracy ↑ & 0.9985 & 0.9985 & 1.0 & 1.0 & 1.0 & 1.0 & 1.0 & 1.0 & 0.9995 & 1.0 & 1.0 & 1.0 & 1.0 & 1.0 \\
    AmazonCoBuy~\citep{amazoncobuy2015} & Node Cl. & Accuracy ↑ & 0.3751 & 0.3778 & 0.9777 & 0.9871 & 0.9996 & 1.0 & 1.0 & 1.0 & 0.9903 & 0.9903 & 0.9903 & 1.0 & 1.0 & 1.0 \\
    BZR\_MD~\citep{morris2020tudatasets} & Graph Cl. & Accuracy ↑ & 0.6503 & 0.6503 & 0.6503 & 0.6503 & 0.8922 & 0.8922 & 0.8922 & 0.8922 & 1.0 & 1.0 & 1.0 & 1.0 & 1.0 & 1.0 \\
    BZR~\citep{morris2020tudatasets} & Graph Cl. & Accuracy ↑ & 0.8198 & 0.9432 & 0.9728 & 0.9778 & 1.0 & 1.0 & 1.0 & 1.0 & - & - & - & - & - & - \\
    CIFAR10~\citep{dwivedi2022benchmarking} & Graph Cl. & Accuracy ↑ & 0.1507 & 1.0 & 1.0 & 1.0 & 1.0 & 1.0 & 1.0 & 1.0 & 1.0 & 1.0 & 1.0 & 1.0 & 1.0 & 1.0 \\
    Citeseer~\citep{coraciteseer2008} & Node Cl. & Accuracy ↑ & 0.2107 & 0.2438 & 0.4968 & 0.7151 & 0.9994 & 0.9997 & 0.9997 & 0.9997 & - & - & - & - & - & - \\
    CLUSTER~\citep{dwivedi2022benchmarking} & Node Cl. & Accuracy ↑ & 0.9488 & 0.9488 & 1.0 & 1.0 & 1.0 & 1.0 & 1.0 & 1.0 & - & - & - & - & - & - \\
    CoauthorCS~\citep{shchur2018pitfalls} & Node Cl. & Accuracy ↑ & 0.2256 & 0.2298 & 0.8756 & 0.9914 & 1.0 & 1.0 & 1.0 & 1.0 & 1.0 & 1.0 & 1.0 & 1.0 & 1.0 & 1.0 \\
    CoauthorPhysics~\citep{shchur2018pitfalls} & Node Cl. & Accuracy ↑ & 0.5052 & 0.5077 & 0.9493 & 0.9979 & 1.0 & 1.0 & 1.0 & 1.0 & 1.0 & 1.0 & 1.0 & 1.0 & 1.0 & 1.0 \\
    COIL-DEL~\citep{morris2020tudatasets} & Graph Cl. & Accuracy ↑ & 0.1151 & 0.7992 & 0.8849 & 0.8851 & 1.0 & 1.0 & 1.0 & 1.0 & 0.9446 & 0.9503 & 0.9503 & 1.0 & 1.0 & 1.0 \\
    COIL-RAG~\citep{morris2020tudatasets} & Graph Cl. & Accuracy ↑ & 0.0444 & 0.0851 & 0.0851 & 0.0851 & 1.0 & 1.0 & 1.0 & 1.0 & 0.88 & 0.88 & 0.88 & 1.0 & 1.0 & 1.0 \\
    COLLAB~\citep{morris2020tudatasets} & Graph Cl. & Accuracy ↑ & 0.607 & 0.9842 & 0.9844 & 0.9844 & - & - & - & - & - & - & - & - & - & - \\
    Cora~\citep{coraciteseer2008} & Node Cl. & Accuracy ↑ & 0.3021 & 0.3143 & 0.764 & 0.9465 & 1.0 & 1.0 & 1.0 & 1.0 & - & - & - & - & - & - \\
    COX2\_MD~\citep{morris2020tudatasets} & Graph Cl. & Accuracy ↑ & 0.5776 & 0.5776 & 0.5776 & 0.5776 & 0.901 & 0.901 & 0.901 & 0.901 & 1.0 & 1.0 & 1.0 & 1.0 & 1.0 & 1.0 \\
    COX2~\citep{morris2020tudatasets} & Graph Cl. & Accuracy ↑ & 0.7923 & 0.833 & 0.9101 & 0.9379 & 1.0 & 1.0 & 1.0 & 1.0 & - & - & - & - & - & - \\
    Cuneiform~\citep{morris2020tudatasets} & Graph Cl. & Accuracy ↑ & 0.206 & 0.206 & 0.206 & 0.206 & 1.0 & 1.0 & 1.0 & 1.0 & 1.0 & 1.0 & 1.0 & 1.0 & 1.0 & 1.0 \\
    DD~\citep{morris2020tudatasets} & Graph Cl. & Accuracy ↑ & 0.837 & 1.0 & 1.0 & 1.0 & 1.0 & 1.0 & 1.0 & 1.0 & - & - & - & - & - & - \\
    DHFR\_MD~\citep{morris2020tudatasets} & Graph Cl. & Accuracy ↑ & 0.7023 & 0.7023 & 0.7023 & 0.7023 & 0.8702 & 0.8702 & 0.8702 & 0.8702 & 1.0 & 1.0 & 1.0 & 1.0 & 1.0 & 1.0 \\
    DHFR~\citep{morris2020tudatasets} & Graph Cl. & Accuracy ↑ & 0.6587 & 0.8585 & 0.9259 & 0.9656 & 1.0 & 1.0 & 1.0 & 1.0 & - & - & - & - & - & - \\
    ENZYMES~\citep{morris2020tudatasets} & Graph Cl. & Accuracy ↑ & 0.385 & 0.9883 & 1.0 & 1.0 & 1.0 & 1.0 & 1.0 & 1.0 & - & - & - & - & - & - \\
    ER\_MD~\citep{morris2020tudatasets} & Graph Cl. & Accuracy ↑ & 0.7018 & 0.7018 & 0.7018 & 0.7018 & 0.8655 & 0.8655 & 0.8655 & 0.8655 & 1.0 & 1.0 & 1.0 & 1.0 & 1.0 & 1.0 \\
    Fingerprint~\citep{morris2020tudatasets} & Graph Cl. & Accuracy ↑ & 0.5186 & 0.5304 & 0.5343 & 0.5371 & 0.9257 & 0.9257 & 0.9257 & 0.9257 & 0.8596 & 0.8596 & 0.8596 & 0.9257 & 0.9257 & 0.9257 \\
    FRANKENSTEIN~\citep{morris2020tudatasets} & Graph Cl. & Accuracy ↑ & 0.6338 & 0.7812 & 0.8891 & 0.8955 & 1.0 & 1.0 & 1.0 & 1.0 & - & - & - & - & - & - \\
    IMDB-BINARY~\citep{morris2020tudatasets} & Graph Cl. & Accuracy ↑ & 0.606 & 0.886 & 0.896 & 0.896 & - & - & - & - & - & - & - & - & - & - \\
    IMDB-MULTI~\citep{morris2020tudatasets} & Graph Cl. & Accuracy ↑ & 0.4413 & 0.6327 & 0.6393 & 0.6393 & - & - & - & - & - & - & - & - & - & - \\
    MNIST~\citep{dwivedi2022benchmarking} & Graph Cl. & Accuracy ↑ & 0.2075 & 0.9999 & 1.0 & 1.0 & 1.0 & 1.0 & 1.0 & 1.0 & 1.0 & 1.0 & 1.0 & 1.0 & 1.0 & 1.0 \\
    MSRC\_9~\citep{morris2020tudatasets} & Graph Cl. & Accuracy ↑ & 0.3122 & 1.0 & 1.0 & 1.0 & 1.0 & 1.0 & 1.0 & 1.0 & - & - & - & - & - & - \\
    Mutagenicity~\citep{morris2020tudatasets} & Graph Cl. & Accuracy ↑ & 0.6255 & 0.8427 & 0.9403 & 0.9555 & 0.923 & 0.9813 & 0.9988 & 1.0 & 0.9175 & 0.9592 & 0.9744 & 0.982 & 0.9988 & 1.0 \\
    MUTAG~\citep{morris2020tudatasets} & Graph Cl. & Accuracy ↑ & 0.8617 & 0.9149 & 0.9628 & 0.9681 & 0.9309 & 0.9574 & 0.9947 & 1.0 & 0.9362 & 0.9894 & 0.9947 & 0.9574 & 0.9947 & 1.0 \\
    NCI109~\citep{morris2020tudatasets} & Graph Cl. & Accuracy ↑ & 0.6346 & 0.8551 & 0.9891 & 0.9932 & 0.9166 & 0.9973 & 0.9988 & 0.999 & - & - & - & - & - & - \\
    NCI1~\citep{morris2020tudatasets} & Graph Cl. & Accuracy ↑ & 0.637 & 0.8521 & 0.9922 & 0.9942 & 0.9134 & 0.9954 & 0.9981 & 0.9983 & - & - & - & - & - & - \\
    ogbg-molbbbp~\citep{hu2020ogb} & Graph Cl. & Accuracy ↑ & 0.7916 & 0.9225 & 0.9799 & 0.9819 & 0.9951 & 0.9956 & 0.9956 & 0.9956 & 0.9848 & 0.9872 & 0.9872 & 0.9956 & 0.9956 & 0.9956 \\
    ogbg-molesol~\citep{hu2020ogb} & Graph Reg. & MSE ↓ & 2.5357 & 1.1065 & 0.7343 & 0.7208 & 0.0215 & 0.0017 & 0.0005 & 0.0005 & 0.4403 & 0.3907 & 0.3859 & 0.0016 & 0.0005 & 0.0005 \\
    ogbg-molfreesolv~\citep{hu2020ogb} & Graph Reg. & MSE ↓ & 12.3188 & 7.5293 & 6.434 & 6.3917 & 0.0877 & 0.0002 & 0.0002 & 0.0002 & 3.2991 & 3.0111 & 3.0111 & 0.0 & 0.0 & 0.0 \\
    ogbg-molhiv~\citep{hu2020ogb} & Graph Cl. & Accuracy ↑ & 0.9652 & 0.9744 & 0.9915 & 0.9932 & 0.9969 & 0.9997 & 1.0 & 1.0 & 0.9928 & 0.9965 & 0.9966 & 0.9997 & 1.0 & 1.0 \\
    ogbg-mollipo~\citep{hu2020ogb} & Graph Reg. & MSE ↓ & 1.3464 & 0.8129 & 0.0816 & 0.0673 & 0.0099 & 0.0005 & 0.0002 & 0.0 & 0.0674 & 0.0415 & 0.0403 & 0.0005 & 0.0002 & 0.0 \\
    ogbn-arxiv~\citep{hu2020ogb} & Node Cl. & Accuracy ↑ & 0.1613 & 0.18 & 0.7682 & 0.9568 & 0.9998 & 1.0 & 1.0 & 1.0 & - & - & - & - & - & - \\
    PATTERN~\citep{dwivedi2022benchmarking} & Node Cl. & Accuracy ↑ & 0.3341 & 0.3365 & 1.0 & 1.0 & 1.0 & 1.0 & 1.0 & 1.0 & - & - & - & - & - & - \\
    PROTEINS~\citep{morris2020tudatasets} & Graph Cl. & Accuracy ↑ & 0.7323 & 0.9524 & 0.9748 & 0.9748 & 1.0 & 1.0 & 1.0 & 1.0 & - & - & - & - & - & - \\
    PTC\_FM~\citep{morris2020tudatasets} & Graph Cl. & Accuracy ↑ & 0.6619 & 0.8367 & 0.894 & 0.894 & 0.9169 & 0.9599 & 0.9828 & 0.9828 & 0.9226 & 0.9542 & 0.9542 & 0.9742 & 0.9971 & 0.9971 \\
    PTC\_FR~\citep{morris2020tudatasets} & Graph Cl. & Accuracy ↑ & 0.6923 & 0.8661 & 0.9259 & 0.9259 & 0.9288 & 0.9772 & 0.9886 & 0.9886 & 0.9402 & 0.963 & 0.963 & 0.9858 & 0.9972 & 0.9972 \\
    PTC\_MM~\citep{morris2020tudatasets} & Graph Cl. & Accuracy ↑ & 0.6696 & 0.8512 & 0.9048 & 0.9048 & 0.9256 & 0.9732 & 0.9881 & 0.9881 & 0.9345 & 0.9554 & 0.9554 & 0.9851 & 0.997 & 0.997 \\
    PTC\_MR~\citep{morris2020tudatasets} & Graph Cl. & Accuracy ↑ & 0.6512 & 0.8401 & 0.8983 & 0.8983 & 0.9186 & 0.9797 & 0.9913 & 0.9913 & 0.907 & 0.939 & 0.939 & 0.9826 & 0.9942 & 0.9942 \\
    Pubmed~\citep{pubmed2012} & Node Cl. & Accuracy ↑ & 0.3994 & 0.4134 & 0.6901 & 0.9251 & 0.9999 & 1.0 & 1.0 & 1.0 & - & - & - & - & - & - \\
    REDDIT-BINARY~\citep{morris2020tudatasets} & Graph Cl. & Accuracy ↑ & 0.8385 & 1.0 & 1.0 & 1.0 & - & - & - & - & - & - & - & - & - & - \\
    REDDIT-MULTI-5K~\citep{morris2020tudatasets} & Graph Cl. & Accuracy ↑ & 0.5517 & 1.0 & 1.0 & 1.0 & - & - & - & - & - & - & - & - & - & - \\
    ZINC~\citep{dwivedi2022benchmarking} & Graph Reg. & MSE ↓ & 3.7182 & 2.9306 & 0.2457 & 0.0251 & 0.5056 & 0.0 & 0.0 & 0.0 & 0.5779 & 0.0027 & 0.0021 & 0.0 & 0.0 & 0.0 \\
    \bottomrule
    \end{tabular}%
    }
\vskip -0.15in
\end{table*}

In this section, we introduce \textsc{Graphtester}, our theoretical framework for analyzing graph datasets in terms of the GNN and graph transformers' capability to model them. We describe the methodology used to compute the score upper bounds for varying number of layers, and how we assess the impact of positional encodings and edge features on the graph transformers' performance. We analyze over 40 graph datasets (refer to Table~\ref{tab:dataset_analysis}) and provide insights into the maximum achievable metrics for different tasks, considering both the presence and absence of original node/edge 
features.

\subsection{Python Package: \textsc{Graphtester}}
\label{sec:python_package}

We provide \textsc{Graphtester} to the research community in the form of a Python package that automates the process of evaluating graph datasets and positional encodings for their processing. This package provides an easy-to-use interface for practitioners, allowing them to perform in-depth analysis of their datasets and make informed decisions on the design of GNNs and graph transformers for their tasks.

\subsection{Data loading}
\label{sec:data_loading}

\textsc{Graphtester} admits datasets in various different formats such as PyG \citep{pyg2019}, DGL \citep{wang2019dgl}, or simply a list of NetworkX \citep{networkx2008} graphs with associated labels. Internally, \textsc{Graphtester} converts them to igraph \citep{igraph2006} objects to efficiently run various isomorphism and labeling algorithms. The datasets analyzed in this paper can be simply loaded by their names.


\subsection{Running 1-WL Algorithm}
\label{sec:graphtester_1wl}

After preprocessing the input dataset and converting it to internal \textit{Dataset} format, \textsc{Graphtester} is able to run 1-WLE on all the graphs in the dataset efficiently until convergence. The final node labels then can be used to create graph-level, link-level and node-level hashes that stays stable across the dataset.


In addition to 1-WL, \textsc{Graphtester} can run $k$-WL algorithm for any $k\leq6$. To the knowledge of the authors, there are no other functional open source implementations of $k$-WL that takes $k$ as an input parameter. \textsc{Graphtester} can also run the folklore variant of $k$-WL, which is more expressive than the default variant.

\subsection{Computing Score Upper Bounds}
\label{sec:computing_bounds}

For calculating the upper score bounds associated with a specified number of layers, we utilize the congruence of Graph Neural Networks (GNNs) and graph transformers with the 1-Weisfeiler-Lehman test (1-WLE), as established earlier in the paper. For each dataset and 1-WLE iteration $k\geq1$, \textsc{Graphtester} determines the hash values for every node within a particular dataset. Through these node hashes, graph hash can calculated by hashing the lexicographically sorted hash of all nodes in a graph. Similarly, link hashes are estimated from lexicographically sorted hashes of interconnected nodes.

For classification measures such as F1 score and accuracy, \textsc{Graphtester} assigns a label to each node, graph or link hash based on the majority rule, following the methodology outlined by Zopf \citep{zopf20221wl}. For regression measures like Mean Square Error (MSE), it assigns the value that minimizes the estimate, which most frequently happens to be the mean. A dataset can be loaded and evaluated through this approach in only a couple of lines of code.

\begin{lstlisting}[language=Python]
import graphtester as gt
dataset = gt.load("ZINC")
evaluation = gt.evaluate(dataset)
print(evaluation.as_dataframe())
\end{lstlisting}

See Table~\ref{tab:dataset_analysis} for the estimation of maximum achievable target metrics for different datasets that has varying tasks and targets. We have estimated the best scores for these datasets in the presence of node and edge features separately by running the 1-WLE algorithm for up to 3 iterations. Considering that 1-WL converges in 2 iterations for nearly all graphs \citep{Kiefer2020WL}, we believe that our results paint a near-definitive picture on what can be theoretically achievable in these datasets.

Overall, a considerable number of datasets in the literature appear to be non-fully-solvable for the target metrics. For the ones that are solvable, often more than a single layer is required. One other important observation here is the need for using available edge features in 1-WL context to achieve better upper bounds --- note that this is not the same as combining edge features with node features as a pre-processing step and running a GNN over it as required for some architectures that do not natively admit edge features \citep{xu2018powerful}, as shown in Theorem \ref{thm:reduction_insufficiency}. 

\begin{figure}[!bth]
    \centering
    \includegraphics[width=0.45\textwidth]{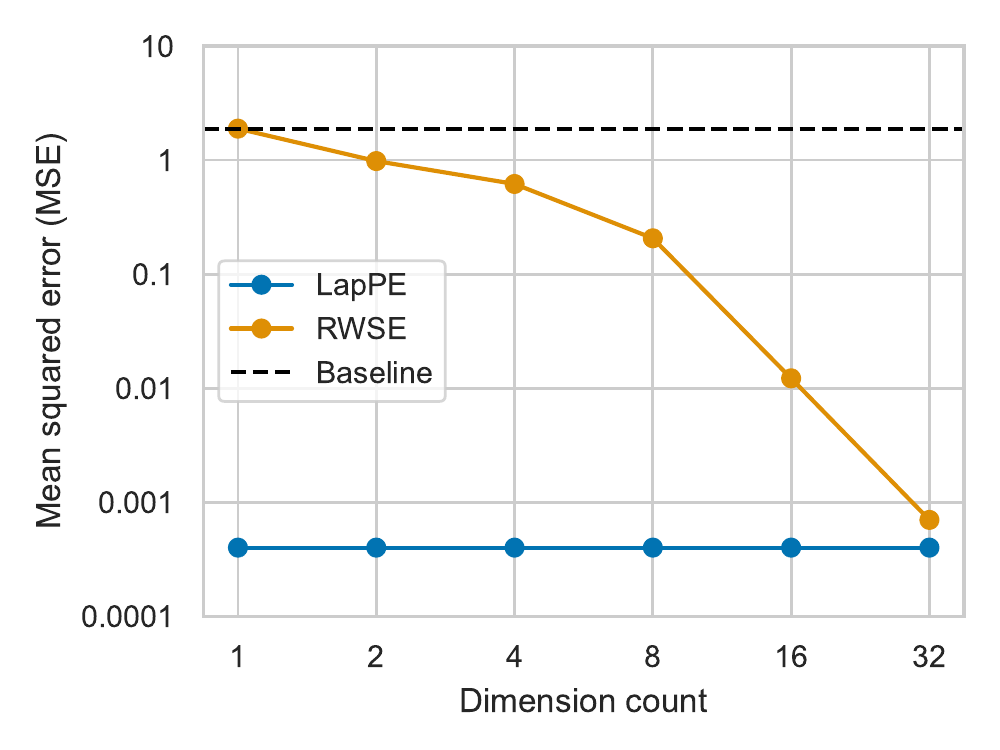}
    \vskip -0.1in
    \caption{Feature evaluation results for \textbf{ZINC} dataset from \textsc{Graphtester} framework for RWSE and LapPE encoding methods. "Dimension count" refers to the number of eigenvectors used for LapPE method, and the walk steps evaluated in RWSE, both of which corresponds to the positional encoding dimensionality. The values are Mean Squared Error, for each node hash mapping to their graph's label. Lower is better.}
    \label{fig:lappe-vs-rwse}
    \vskip -0.2in
\end{figure}

\subsection{Assessing the Impact of Additional Features}
\label{sec:assessing_impact}

Having evaluated the datasets with respect to the application of available node and edge features, an ensuing question emerges: Can these upper boundaries be enhanced, potentially improving overall GNN/Transformer performances on such graph datasets, by incorporating deterministic, pre-computed metrics derived from the literature? These metrics have been examined in the context of both GNNs and Transformers, most notably as Subgraph Counting \citep{bouritsas2020subgraphcounting} in the case of the former, and the positional encoding concept for the latter.

For this target, \textsc{Graphtester} offers an interface for researchers to try out various metrics and their combinations in the context of 1-WLE test to assess the potential upper bound improvements. A straightforward use case for this interface is to answer the question "how many dimensions of positional encoding do I need for my dataset?". 

To answer this question in the context of ZINC dataset, we analyzed the methods random-walk structural encoding (RWSE) and Laplacian eigenvectors encoding (LapPE) in identifiability of each graph through individual node hashes. The results can be seen in Figure \ref{fig:lappe-vs-rwse}. It is somewhat surprising to note that although RWSE highly benefits from increased step counts, LapPE provides optimal results even in a single iteration for ZINC dataset.

Another possible use case for feature evaluation might appear to be to choose the encoding method that brings the most benefits in the target upper bound score. However, such use of 1-WL framework to improve GNN performance does not seem to have a strong basis. Indeed, Figure \ref{fig:lappe-vs-rwse} is a counter-example here: although LapPE provides better upper bound scores in node-based matching of graphs to their labels, RWSE have proven to be the better encoding in GraphGPS study \citep{yuan2021graphgps}. A possible reason for the mismatch is also noted in the referenced work, and is sourced from the domain mismatch of LapLE in molecular datasets. It is argued that improved node-based identifiability only leads the network to overfit to the spurious variance introduced by the positional encoding. We recommend to choose the encoding method according to the domain of the task, but possibly tune the parameters of the encoding via \textsc{Graphtester}, ideally choosing the minimal encoding that provides a sufficient level of identifiability.

\subsection{Injecting Arbitrary Features into Training Pipelines}
\label{sec:injecting_features}

After analyzing a dataset and evaluating potential positional encodings, labeling methods implemented in \textsc{Graphtester} package can easily be embedded into the training pipeline. For this purpose, we expose a pretransform method in the package itself, that can admit different feature names, and be provided to PyG dataset objects to transform datasets before the training.


\begin{table}[!bt]
  \caption{Performance of seven simple \textsc{Graphtester} features when used as positional encodings for GraphGPS. We keep the same setup as the best GPS model, only changing the positional encodings. Mean and standard deviation are reported over five runs each.}
  \label{tab:gps_performance}
  \centering
  \resizebox{0.5\textwidth}{!}{%
  \begin{tabular}{lccccc}
    \toprule
    \textbf{Feature}& \textbf{MNIST $\uparrow$} & \textbf{CIFAR10 $\uparrow$} & \textbf{PATTERN $\uparrow$}  \\
    \midrule
    Best GPS        & 98.051 $\pm$ 0.126 & 72.298 $\pm$ 0.356 & 86.685 $\pm$ 0.059\\
    \midrule
    Local transitivity& 98.016 $\pm$  0.054 & 72.466 $\pm$  0.273 & 85.979 $\pm$  0.179 \\
Eccentricity& 98.030 $\pm$  0.148 & 71.830 $\pm$  0.695 & 85.805 $\pm$  0.405 \\
Eigenvector centrality& 97.996 $\pm$  0.084 & 72.514 $\pm$  0.268 & 86.234 $\pm$  0.215\\
Burt's constraint& 97.890 $\pm$  0.072 & 72.456 $\pm$  0.421 & 86.302 $\pm$  0.150 \\
Closeness centrality& 97.890 $\pm$  0.070 & 73.006 $\pm$  0.380 & 85.959 $\pm$  0.077 \\
Betweenness centrality& 97.956 $\pm$  0.067 & 72.140 $\pm$  0.538 & 86.267 $\pm$  0.288 \\
Harmonic centrality& 97.974 $\pm$  0.156 & 72.282 $\pm$  0.314 & 85.889 $\pm$  0.060 \\

    \bottomrule
  \end{tabular}
  }
\vskip -0.1in
\end{table}

\subsection{Performance on Real-World Datasets}
\label{sec:gps}

\textsc{Graphtester} implements various classical centrality metrics and positional encoding methods for graph dataset analysis and pre-transform. To demonstrate that \textsc{Graphtester} features can be straightforwardly used in practice and even the simplest centrality metrics may carry value, we add them as positional encodings to GraphGPS~\citep{rampavsek2022recipe_graphGPS} and measure the performance of the resulting model. After the features are generated for nodes, we apply one linear layer to encode them together with input node features. Results for seven \textsc{Graphtester} features are summarized in Table~\ref{tab:gps_performance}. We can observe that all tested features can get relatively close to the best encodings on MNIST, CIFAR10, and PATTERN, even beating the best encodings on CIFAR10.

\subsection{A Synthetic Dataset for Benchmarking of Node and Edge Features}
\label{sec:dataset}

Additional features provide a way for GNNs and GTs to identify nodes in a graph. Our findings reinforce the efficacy of these methods in enhancing node and graph identifiability. Nevertheless, there remains a notable gap in the existing literature: the absence of a dataset specifically designed to benchmark features such as positional encoding methods based on their performance across different graph classes.

As the final contribution of this work, we address this shortfall by introducing a synthetic graph dataset that is provided as part of the \textsc{Graphtester} package. This dataset uniquely serves as a rigorous testing ground for the effectiveness of node and edge pre-coloring methods within the 1-Weisfeiler-Lehman (1-WL) framework. Using \textsc{Graphtester} framework, researchers can label this dataset with an arbitrary feature encoding of their own, and evaluate it to acquire its comparative standing with respect to other pre-coloring methods in the literature, as well as $k$-WL test for $k\geq2$.

An overview of the graph classes contained in the \textsc{Graphtester} dataset, as well as the definitions of these graph classes can be found in Appendix \ref{sec:graph_classes}. As a baseline, we evaluated the methods available in the \textsc{Graphtester} framework against the dataset. We report the results for noteworthy graph classes and pre-coloring methods in Table \ref{tab:gt_dataset_results}. Refer to the Appendix \ref{sec:additional_results} for a discussion on some of the results in this table, and how to overcome the difficulties of identifying graphs in 1-WL framework.

\begin{table}[!tb]
\centering
\caption{Failure counts of all tests conducted for all graph pairs in the given graph classes with noted orders. Only noteworthy graph classes and methods are listed.}
\label{tab:gt_dataset_results}
\resizebox{0.48\textwidth}{!}{%

\begin{tabular}{@{}c|ccc|cccccc|ccccccc@{}}
\toprule
\multicolumn{1}{r|}{\textit{graph class:}} & \multicolumn{3}{c|}{All} & \multicolumn{6}{c|}{Highly irregular} & \multicolumn{7}{c}{Strongly regular} \\ \cmidrule(l){2-17} 
\multicolumn{1}{r|}{\textit{node count:}}  & 6    & 7    & 8   & 8   & 9   & 10   & 11   & 12   & 13   & 16  & 25  & 26  & 28  & 29 & 36 & 40 \\ \midrule
1-WL ($\equiv$ 2-WL)                        & 4 & 22 & 350 & 1 & 0 & 8 & 0 & 165 & 0 & 1 & 105 & 45 & 6 & 820 & 16110 & 378 \\
3-WL                      & 0 & 0  & 0   & 0 & 0 & 0 & 0 & 0   & 0 & 1 & 105 & 45 & 6 & 820 & 16110 & 378 \\
4-WL                      & 0 & 0  & 0   & 0 & 0 & 0 & 0 & 0   & 0 & 0 & 0   & 0  & 0 & 0   & 0     & 0   \\ \midrule
Closeness centrality      & 2 & 14 & 202 & 1 & 0 & 2 & 0 & 32  & 0 & 1 & 105 & 45 & 6 & 820 & 16110 & 378 \\
Eigenvector centrality    & 3 & 20 & 335 & 1 & 0 & 8 & 0 & 165 & 0 & 1 & 105 & 45 & 6 & 820 & 16110 & 378 \\
Harmonic centrality       & 2 & 14 & 202 & 1 & 0 & 2 & 0 & 32  & 0 & 1 & 105 & 45 & 6 & 820 & 16110 & 378 \\
Betweenness centrality    & 1 & 2  & 16  & 0 & 0 & 0 & 0 & 0   & 0 & 1 & 105 & 45 & 6 & 820 & 16110 & 378 \\
Eccentricity              & 3 & 16 & 243 & 1 & 0 & 6 & 0 & 136 & 0 & 1 & 105 & 45 & 6 & 820 & 16110 & 378 \\
Local transitivity        & 0 & 0  & 21  & 0 & 0 & 1 & 0 & 3   & 0 & 1 & 105 & 45 & 6 & 820 & 16110 & 378 \\
Burt's constraint         & 0 & 0  & 6   & 0 & 0 & 1 & 0 & 3   & 0 & 1 & 105 & 45 & 6 & 820 & 16110 & 378 \\
Edge betweenness          & 0 & 0  & 0   & 0 & 0 & 0 & 0 & 0   & 0 & 1 & 105 & 45 & 6 & 820 & 16110 & 378 \\
Convergence degree        & 2 & 8  & 57  & 0 & 0 & 0 & 0 & 0   & 0 & 1 & 105 & 45 & 6 & 820 & 16110 & 378 \\
4-clique count of edges   & 4 & 18 & 247 & 1 & 0 & 8 & 0 & 137 & 0 & 0 & 0   & 2  & 0 & 0   & 309   & 4   \\
5-path count of edges     & 0 & 0  & 0   & 0 & 0 & 0 & 0 & 0   & 0 & 1 & 105 & 45 & 6 & 820 & 16110 & 378 \\
6-path count of vertices  & 0 & 0  & 1   & 0 & 0 & 0 & 0 & 0   & 0 & 1 & 105 & 45 & 6 & 820 & 16110 & 378 \\
6-cycle count of vertices & 0 & 1  & 8   & 0 & 0 & 0 & 0 & 3   & 0 & 1 & 105 & 45 & 6 & 820 & 16110 & 378 \\ \bottomrule
\end{tabular}%
}
\vskip -0.1in
\end{table}

\section{Conclusion}
\label{sec:conclusion}

This paper introduces \textsc{Graphtester}, a powerful tool designed for in-depth theoretical analysis of Graph Neural Networks (GNNs) and graph transformers. \textsc{Graphtester} has demonstrated its capability to discern the upper bounds of performance across various datasets, taking into account aspects like edge features, different performance measures, varying numbers of GNN layers, and higher-order Weisfeiler-Lehman (WL) tests. Together with the package, we make public a 55,000-graph synthetic dataset for the purpose of benchmarking positional encoding methods, that contains many graphs that are hard to distinguish in $k$-WL framework.

We have also established critical theoretical insights regarding GNNs and graph transformers, proving that the latter's power is bounded by the 1-WL test if positional encodings are only used for nodes, and placed them on a theoretical basis in the presence of edge features. This underscores the fundamental role of positional encodings in amplifying the expressive power of these models.

A key aspect of our work has been the comprehensive analysis of over 40 graph datasets from the literature using \textsc{Graphtester}. This extensive study has revealed that not all datasets are fully solvable with the tasks at hand, pointing to inherent complexities in graph data that may challenge even state-of-the-art GNN architectures. Furthermore, we found that even when a dataset is theoretically solvable, the effective use of edge features is vital. Our theoretical analysis underscores that edge features, when appropriately incorporated, can substantially enhance the expressiveness and performance of GNNs.

Overall, \textsc{Graphtester} not only advances our theoretical understanding of GNNs and graph transformers but also offers practical guidance for their optimal deployment across a variety of tasks and datasets. Future work will aim to extend the capabilities of \textsc{Graphtester} to accommodate different graph dataset formats and tasks, and delve deeper into the role of positional encodings.

\clearpage

\bibliography{main}

\begin{thebibliography}{52}
\providecommand{\natexlab}[1]{#1}
\providecommand{\url}[1]{\texttt{#1}}
\expandafter\ifx\csname urlstyle\endcsname\relax
  \providecommand{\doi}[1]{doi: #1}\else
  \providecommand{\doi}{doi: \begingroup \urlstyle{rm}\Url}\fi

\bibitem[Barth{\'e}lemy(2011)]{barthelemy2011spatial}
Barth{\'e}lemy, M.
\newblock Spatial networks.
\newblock \emph{Physics Reports}, 499\penalty0 (1-3):\penalty0 1--101, 2011.

\bibitem[Bouritsas et~al.(2020)Bouritsas, Frasca, Zafeiriou, and
  Bronstein]{bouritsas2020subgraphcounting}
Bouritsas, G., Frasca, F., Zafeiriou, S., and Bronstein, M.~M.
\newblock Improving graph neural network expressivity via subgraph isomorphism
  counting.
\newblock \emph{CoRR}, abs/2006.09252, 2020.
\newblock URL \url{https://arxiv.org/abs/2006.09252}.

\bibitem[Brandes(2004)]{2004-BrandesEB}
Brandes, U.
\newblock A faster algorithm for betweenness centrality.
\newblock \emph{The Journal of Mathematical Sociology}, 25, 03 2004.
\newblock \doi{10.1080/0022250X.2001.9990249}.

\bibitem[Bronstein et~al.(2017)Bronstein, Bruna, LeCun, Szlam, and
  Vandergheynst]{bronstein2017geometric}
Bronstein, M.~M., Bruna, J., LeCun, Y., Szlam, A., and Vandergheynst, P.
\newblock Geometric deep learning: going beyond euclidean data.
\newblock \emph{IEEE Signal Processing Magazine}, 34\penalty0 (4):\penalty0
  18--42, 2017.

\bibitem[Brouwer \& Van~Maldeghem(2022)Brouwer and
  Van~Maldeghem]{2022-BrouwerStronglyReg}
Brouwer, A.~E. and Van~Maldeghem, H.
\newblock \emph{Strongly Regular Graphs}.
\newblock Encyclopedia of Mathematics and its Applications. Cambridge
  University Press, 2022.
\newblock \doi{10.1017/9781009057226}.

\bibitem[Cai et~al.(1989)Cai, Furer, and Immerman]{cfi1989cr}
Cai, J.-Y., Furer, M., and Immerman, N.
\newblock An optimal lower bound on the number of variables for graph
  identification.
\newblock In \emph{30th Annual Symposium on Foundations of Computer Science},
  pp.\  612--617, 1989.
\newblock \doi{10.1109/SFCS.1989.63543}.

\bibitem[Chen et~al.(2020)Chen, Chen, Villar, and
  Bruna]{chen2020canSubgraphCounting}
Chen, Z., Chen, L., Villar, S., and Bruna, J.
\newblock Can graph neural networks count substructures?
\newblock \emph{Advances in neural information processing systems},
  33:\penalty0 10383--10395, 2020.

\bibitem[Csardi \& Nepusz(2006)Csardi and Nepusz]{igraph2006}
Csardi, G. and Nepusz, T.
\newblock The igraph software package for complex network research.
\newblock \emph{InterJournal}, Complex Systems:\penalty0 1695, 2006.
\newblock URL \url{https://igraph.org}.

\bibitem[Dobson \& Doig(2003)Dobson and Doig]{dobson2003distinguishing}
Dobson, P.~D. and Doig, A.~J.
\newblock Distinguishing enzyme structures from non-enzymes without alignments.
\newblock \emph{Journal of molecular biology}, 330\penalty0 (4):\penalty0
  771--783, 2003.

\bibitem[Duvenaud et~al.(2015)Duvenaud, Maclaurin, Iparraguirre, Bombarell,
  Hirzel, Aspuru-Guzik, and Adams]{duvenaud2015convolutional}
Duvenaud, D.~K., Maclaurin, D., Iparraguirre, J., Bombarell, R., Hirzel, T.,
  Aspuru-Guzik, A., and Adams, R.~P.
\newblock Convolutional networks on graphs for learning molecular fingerprints.
\newblock In \emph{Advances in neural information processing systems},
  volume~28, pp.\  2224--2232, 2015.

\bibitem[Dwivedi \& Bresson(2020)Dwivedi and
  Bresson]{dwivedi2020generalization}
Dwivedi, V.~P. and Bresson, X.
\newblock A generalization of transformer networks to graphs.
\newblock \emph{arXiv preprint arXiv:2012.09699}, 2020.

\bibitem[Dwivedi et~al.(2020)Dwivedi, Joshi, Laurent, Bengio, and
  Bresson]{dwivedi2020benchmarking}
Dwivedi, V.~P., Joshi, C.~K., Laurent, T., Bengio, Y., and Bresson, X.
\newblock Benchmarking graph neural networks.
\newblock In \emph{Advances in Neural Information Processing Systems},
  volume~33, 2020.

\bibitem[Dwivedi et~al.(2021)Dwivedi, Luu, Laurent, Bengio, and
  Bresson]{dwivedi2021graph}
Dwivedi, V.~P., Luu, A.~T., Laurent, T., Bengio, Y., and Bresson, X.
\newblock Graph neural networks with learnable structural and positional
  representations.
\newblock \emph{arXiv preprint arXiv:2110.07875}, 2021.

\bibitem[Dwivedi et~al.(2022)Dwivedi, Joshi, Luu, Laurent, Bengio, and
  Bresson]{dwivedi2022benchmarking}
Dwivedi, V.~P., Joshi, C.~K., Luu, A.~T., Laurent, T., Bengio, Y., and Bresson,
  X.
\newblock Benchmarking graph neural networks, 2022.

\bibitem[Fey \& Lenssen(2019)Fey and Lenssen]{pyg2019}
Fey, M. and Lenssen, J.~E.
\newblock Fast graph representation learning with {PyTorch Geometric}.
\newblock In \emph{ICLR Workshop on Representation Learning on Graphs and
  Manifolds}, 2019.

\bibitem[Gago et~al.(2012)Gago, Hurajová, and Madaras]{2012-GagoBetwRegular}
Gago, S., Hurajová, J., and Madaras, T.
\newblock Betweenness-selfcentric graphs.
\newblock \emph{Betweenness-selfcentric graphs}, Apr 2012.
\newblock URL \url{http://hdl.handle.net/2117/15768}.

\bibitem[Hagberg et~al.(2008)Hagberg, Swart, and S~Chult]{networkx2008}
Hagberg, A., Swart, P., and S~Chult, D.
\newblock Exploring network structure, dynamics, and function using networkx.
\newblock 1 2008.
\newblock URL \url{https://www.osti.gov/biblio/960616}.

\bibitem[Hamilton et~al.(2017)Hamilton, Ying, and
  Leskovec]{hamilton2017inductive}
Hamilton, W.~L., Ying, R., and Leskovec, J.
\newblock Inductive representation learning on large graphs.
\newblock In \emph{Advances in Neural Information Processing Systems}, pp.\
  1024--1034, 2017.

\bibitem[Hu et~al.(2019)Hu, Liu, Gomes, Zitnik, Liang, Pande, and
  Leskovec]{hu2019gineconv}
Hu, W., Liu, B., Gomes, J., Zitnik, M., Liang, P., Pande, V., and Leskovec, J.
\newblock Strategies for pre-training graph neural networks, 2019.
\newblock URL \url{https://arxiv.org/abs/1905.12265}.

\bibitem[Hu et~al.(2020)Hu, Fey, Zitnik, Dong, Ren, Liu, Catasta, and
  Leskovec]{hu2020ogb}
Hu, W., Fey, M., Zitnik, M., Dong, Y., Ren, H., Liu, B., Catasta, M., and
  Leskovec, J.
\newblock Open graph benchmark: Datasets for machine learning on graphs.
\newblock \emph{arXiv preprint arXiv:2005.00687}, 2020.

\bibitem[Immerman \& Lander(1990)Immerman and Lander]{immerman1990cr}
Immerman, N. and Lander, E.
\newblock \emph{Describing Graphs: A First-Order Approach to Graph
  Canonization}, pp.\  59--81.
\newblock Springer New York, New York, NY, 1990.
\newblock ISBN 978-1-4612-4478-3.
\newblock \doi{10.1007/978-1-4612-4478-3_5}.
\newblock URL \url{https://doi.org/10.1007/978-1-4612-4478-3_5}.

\bibitem[Junttila \& Kaski(2007)Junttila and Kaski]{junttila2007bliss}
Junttila, T. and Kaski, P.
\newblock Engineering an efficient canonical labeling tool for large and sparse
  graphs.
\newblock In \emph{Proceedings of the Meeting on Algorithm Engineering \&
  Expermiments}, pp.\  135–149, USA, 2007. Society for Industrial and Applied
  Mathematics.

\bibitem[Kiefer(2020)]{Kiefer2020WL}
Kiefer, S.
\newblock \emph{{P}ower and limits of the {W}eisfeiler-{L}eman algorithm}.
\newblock Dissertation, RWTH Aachen University, Aachen, 2020.
\newblock URL \url{https://publications.rwth-aachen.de/record/785831}.
\newblock Veröffentlicht auf dem Publikationsserver der RWTH Aachen
  University; Dissertation, RWTH Aachen University, 2020.

\bibitem[Kipf \& Welling(2017)Kipf and Welling]{kipf2017semi}
Kipf, T.~N. and Welling, M.
\newblock Semi-supervised classification with graph convolutional networks.
\newblock In \emph{ICLR}, 2017.

\bibitem[Kreuzer et~al.(2021)Kreuzer, Beaini, Hamilton, L{\'e}tourneau, and
  Tossou]{kreuzer2021rethinking}
Kreuzer, D., Beaini, D., Hamilton, W., L{\'e}tourneau, V., and Tossou, P.
\newblock Rethinking graph transformers with spectral attention.
\newblock \emph{Advances in Neural Information Processing Systems},
  34:\penalty0 21618--21629, 2021.

\bibitem[Li et~al.(2018)Li, Tarlow, Brockschmidt, and Zemel]{li2018adaptive}
Li, Y., Tarlow, D., Brockschmidt, M., and Zemel, R.
\newblock Adaptive embedding dimension selection in graph neural networks.
\newblock In \emph{Advances in Neural Information Processing Systems},
  volume~31, pp.\  7458--7468, 2018.

\bibitem[Maron et~al.(2020)Maron, Ben-Hamu, Serviansky, and
  Lipman]{maron2020provably}
Maron, H., Ben-Hamu, H., Serviansky, H., and Lipman, Y.
\newblock Provably powerful graph networks.
\newblock In \emph{Advances in Neural Information Processing Systems},
  volume~33, 2020.

\bibitem[McAuley et~al.(2015)McAuley, Targett, Shi, and van~den
  Hengel]{amazoncobuy2015}
McAuley, J., Targett, C., Shi, Q., and van~den Hengel, A.
\newblock Image-based recommendations on styles and substitutes, 2015.

\bibitem[McKay \& Piperno(2013)McKay and Piperno]{mckay2013nauty}
McKay, B.~D. and Piperno, A.
\newblock Practical graph isomorphism, ii, 2013.
\newblock URL \url{https://arxiv.org/abs/1301.1493}.

\bibitem[Morris et~al.(2019)Morris, Ritzert, Fey, Hamilton, Lenssen, Rattan,
  and Grohe]{morris2019weisfeiler}
Morris, C., Ritzert, M., Fey, M., Hamilton, W.~L., Lenssen, J.~E., Rattan, G.,
  and Grohe, M.
\newblock Weisfeiler and leman go neural: Higher-order graph neural networks.
\newblock \emph{Proceedings of the AAAI Conference on Artificial Intelligence},
  33:\penalty0 4602--4609, 2019.

\bibitem[Morris et~al.(2020)Morris, Kriege, Bause, Kersting, Mutzel, and
  Neumann]{morris2020tudatasets}
Morris, C., Kriege, N.~M., Bause, F., Kersting, K., Mutzel, P., and Neumann, M.
\newblock Tudataset: {A} collection of benchmark datasets for learning with
  graphs.
\newblock \emph{CoRR}, abs/2007.08663, 2020.
\newblock URL \url{https://arxiv.org/abs/2007.08663}.

\bibitem[Namata et~al.(2012)Namata, London, Getoor, and Huang]{pubmed2012}
Namata, G.~M., London, B., Getoor, L., and Huang, B.
\newblock Query-driven active surveying for collective classification.
\newblock In \emph{Workshop on Mining and Learning with Graphs}, 2012.

\bibitem[Newman(2003)]{newman2003structure}
Newman, M.~E.
\newblock The structure and function of complex networks.
\newblock \emph{SIAM review}, 45\penalty0 (2):\penalty0 167--256, 2003.

\bibitem[Papp \& Wattenhofer(2022)Papp and Wattenhofer]{papp2022theoretical}
Papp, P.~A. and Wattenhofer, R.
\newblock A theoretical comparison of graph neural network extensions.
\newblock In \emph{International Conference on Machine Learning}, pp.\
  17323--17345. PMLR, 2022.

\bibitem[Ramp{\'a}{\v{s}}ek et~al.(2022)Ramp{\'a}{\v{s}}ek, Galkin, Dwivedi,
  Luu, Wolf, and Beaini]{rampavsek2022recipe_graphGPS}
Ramp{\'a}{\v{s}}ek, L., Galkin, M., Dwivedi, V.~P., Luu, A.~T., Wolf, G., and
  Beaini, D.
\newblock Recipe for a general, powerful, scalable graph transformer.
\newblock \emph{Advances in Neural Information Processing Systems},
  35:\penalty0 14501--14515, 2022.

\bibitem[Rombach \& Porter(2013)Rombach and
  Porter]{2013-RombachClosenessRegular}
Rombach, M.~P. and Porter, M.~A.
\newblock Discriminating power of centrality measures.
\newblock \emph{CoRR}, abs/1305.3146, 2013.
\newblock URL \url{http://arxiv.org/abs/1305.3146}.

\bibitem[Scarselli et~al.(2009)Scarselli, Gori, Tsoi, Hagenbuchner, and
  Monfardini]{scarselli2009graph}
Scarselli, F., Gori, M., Tsoi, A.~C., Hagenbuchner, M., and Monfardini, G.
\newblock The graph neural network model.
\newblock \emph{IEEE Transactions on Neural Networks}, 20\penalty0
  (1):\penalty0 61--80, 2009.

\bibitem[Schlichtkrull et~al.(2018)Schlichtkrull, Kipf, Bloem, van~den Berg,
  Titov, and Welling]{schlichtkrull2018modeling}
Schlichtkrull, M., Kipf, T.~N., Bloem, P., van~den Berg, R., Titov, I., and
  Welling, M.
\newblock Modeling relational data with graph convolutional networks.
\newblock In \emph{Proceedings of the 15th European Semantic Web Conference},
  pp.\  593--607. Springer, 2018.

\bibitem[Sen et~al.(2008)Sen, Namata, Bilgic, Getoor, Galligher, and
  Eliassi-Rad]{coraciteseer2008}
Sen, P., Namata, G., Bilgic, M., Getoor, L., Galligher, B., and Eliassi-Rad, T.
\newblock Collective classification in network data.
\newblock \emph{AI Magazine}, 29\penalty0 (3):\penalty0 93, Sep. 2008.
\newblock \doi{10.1609/aimag.v29i3.2157}.

\bibitem[Shchur et~al.(2018)Shchur, Mumme, Bojchevski, and
  G{\"u}nnemann]{shchur2018pitfalls}
Shchur, O., Mumme, M., Bojchevski, A., and G{\"u}nnemann, S.
\newblock Pitfalls of graph neural network evaluation.
\newblock \emph{Relational Representation Learning Workshop, NeurIPS 2018},
  2018.

\bibitem[Skopek et~al.(2020)Skopek, Ganea, and
  G{\"u}nnemann]{skopek2020message}
Skopek, O., Ganea, O., and G{\"u}nnemann, S.
\newblock Message passing with poincar\'{e} embeddings and hyperbolic graph
  attention networks.
\newblock In \emph{Proceedings of the Eighth International Conference on
  Learning Representations}, 2020.

\bibitem[Vaswani et~al.(2017)Vaswani, Shazeer, Parmar, Uszkoreit, Jones, Gomez,
  Kaiser, and Polosukhin]{vaswani2017attention}
Vaswani, A., Shazeer, N., Parmar, N., Uszkoreit, J., Jones, L., Gomez, A.~N.,
  Kaiser, L., and Polosukhin, I.
\newblock Attention is all you need.
\newblock In \emph{Advances in neural information processing systems},
  volume~30, pp.\  5998--6008, 2017.

\bibitem[Veli{\v{c}}kovi{\'c} et~al.(2018)Veli{\v{c}}kovi{\'c}, Cucurull,
  Casanova, Romero, Li{\`o}, and Bengio]{velivckovic2018graph}
Veli{\v{c}}kovi{\'c}, P., Cucurull, G., Casanova, A., Romero, A., Li{\`o}, P.,
  and Bengio, Y.
\newblock Graph attention networks.
\newblock In \emph{ICLR}, 2018.

\bibitem[Wang et~al.(2019)Wang, Zheng, Ye, Gan, Li, Song, Zhou, Ma, Yu, Gai,
  Xiao, He, Karypis, Li, and Zhang]{wang2019dgl}
Wang, M., Zheng, D., Ye, Z., Gan, Q., Li, M., Song, X., Zhou, J., Ma, C., Yu,
  L., Gai, Y., Xiao, T., He, T., Karypis, G., Li, J., and Zhang, Z.
\newblock Deep graph library: A graph-centric, highly-performant package for
  graph neural networks.
\newblock \emph{arXiv preprint arXiv:1909.01315}, 2019.

\bibitem[Wu et~al.(2022{\natexlab{a}})Wu, Zhao, Li, Wipf, and
  Yan]{wu2022nodeformer}
Wu, Q., Zhao, W., Li, Z., Wipf, D.~P., and Yan, J.
\newblock Nodeformer: A scalable graph structure learning transformer for node
  classification.
\newblock \emph{Advances in Neural Information Processing Systems},
  35:\penalty0 27387--27401, 2022{\natexlab{a}}.

\bibitem[Wu et~al.(2020)Wu, Pan, Chen, Long, Zhang, and
  Yu]{wu2020comprehensive}
Wu, Z., Pan, S., Chen, F., Long, G., Zhang, C., and Yu, P.~S.
\newblock A comprehensive survey on graph neural networks.
\newblock \emph{IEEE Transactions on Neural Networks and Learning Systems},
  2020.

\bibitem[Wu et~al.(2022{\natexlab{b}})Wu, Jain, Wright, Mirhoseini, Gonzalez,
  and Stoica]{wu2022graphtrans}
Wu, Z., Jain, P., Wright, M.~A., Mirhoseini, A., Gonzalez, J.~E., and Stoica,
  I.
\newblock Representing long-range context for graph neural networks with global
  attention.
\newblock 2022{\natexlab{b}}.

\bibitem[Xu et~al.(2019{\natexlab{a}})Xu, Hu, Leskovec, and
  Jegelka]{xu2018powerful}
Xu, K., Hu, W., Leskovec, J., and Jegelka, S.
\newblock How powerful are graph neural networks?
\newblock In \emph{Proceedings of the International Conference on Learning
  Representations (ICLR)}, 2019{\natexlab{a}}.

\bibitem[Xu et~al.(2019{\natexlab{b}})Xu, Hu, Leskovec, and
  Jegelka]{xu2019powerful}
Xu, K., Hu, W., Leskovec, J., and Jegelka, S.
\newblock How powerful are graph neural networks?
\newblock In \emph{Proceedings of the International Conference on Learning
  Representations}, 2019{\natexlab{b}}.

\bibitem[Ying et~al.(2021)Ying, Cai, Luo, Zheng, Ke, He, Shen, and
  Liu]{ying2021transformers}
Ying, C., Cai, T., Luo, S., Zheng, S., Ke, G., He, D., Shen, Y., and Liu, T.-Y.
\newblock Do transformers really perform badly for graph representation?
\newblock \emph{Advances in Neural Information Processing Systems},
  34:\penalty0 28877--28888, 2021.

\bibitem[Yuan et~al.(2021)Yuan, Huang, and Xu]{yuan2021graphgps}
Yuan, Z., Huang, W., and Xu, K.
\newblock Graphgps: Graph generative pre-training with semantic preserving.
\newblock In \emph{Proceedings of the International Conference on Learning
  Representations}, 2021.

\bibitem[Zopf(2022)]{zopf20221wl}
Zopf, M.
\newblock 1-wl expressiveness is (almost) all you need, 2022.
\newblock URL \url{https://arxiv.org/abs/2202.10156}.

\end{thebibliography}
\bibliographystyle{icml2023}

\newpage
\appendix
\onecolumn

\section{Proofs}
\printProofs

\section{Definition of Various Graph Classes}
\label{sec:graph_classes}

An overview of the graph classes and the number of graph counts per class can be found in Table \ref{tab:gt_dataset}.

\begin{table}[!tbh]
    \caption{Graph classes and counts available in the \textsc{Graphtester} dataset. Except for distance-regular graphs, all graph classes are exhaustive until the given degree.}
    \label{tab:gt_dataset}
    \centering
      \begin{tabular}{lcc}
        \toprule
         Graph class & Max. vertex count & Graph count \\
         \midrule
         All nonisomorphic & 8 & 13595\\
         Eulerian & 9 & 2363\\
         Planar connected & 8 & 6747\\
         Chordal & 9 & 13875\\
         Perfect & 8 & 9974\\
         Highly irregular & 13 & 624\\
         Edge-4-critical & 11 & 1399\\
         Self-complementary & 13 & 6368\\
         Distance-regular & 40 & 115\\
         Strongly regular & 40 & 280\\
    \bottomrule
  \end{tabular}
\vskip -0.1in
\end{table}

\paragraph{All graphs ($|V|<9$)} We consider all graphs below a certain order, without considering any additional properties..

\paragraph{Eulerian graphs ($|V|<10$)} A graph is Eulerian if it contains only even-degree vertices.

\paragraph{Planar connected graphs ($|V|<9$)} A graph is planar if it can be drawn on a plane without any intersecting edges.

\paragraph{Chordal graphs ($|V|<10$)} A graph is chordal if every cycle of length 4 and more has a chord, that is, an edge connecting the non-adjacent nodes of the cycle.

\paragraph{Perfect graphs ($|V|<8$)} A graph is perfect if every odd cycle of length 5 and more has a chord, and the same is true of the complement graph.

\paragraph{Highly irregular graphs ($|V|<14$)} A graph is highly irregular if for every vertex in that graph, all neighbors of that vertex have distinct degrees.

\paragraph{Edge-4-critical graphs ($|V|<12$)} A graph is edge-4-critical if it is connected, is (vertex) 4-colourable, and removal of any edge makes it 3-colourable.

\paragraph{Self-complementary graphs ($|V|<14$)} A graph is self-complement if it is isomorphic to its complement.

\paragraph{Distance-regular graphs ($|V|<41$)} For a graph with order $n$ and diameter $d$, any vertex $u$ and for any integer i with $0\leq i\leq d$, let $G_i(u)$ denote the set of vertices at distance i from $u$. If $v \in G_i(u)$ and $w$ is a neighbour of $v$, then $w$ must be at distance i-1, i or i+1 from $u$. Let $c_i$, $a_i$ and $b_i$ denote the number of such vertices w. A graph is distance-regular if and only if these parameters $c_i$, $a_i$ and $b_i$ depends only on the distance i, and not on the choice of $u$ and $v$. 

\paragraph{Strongly regular graphs ($|V|<41$)} A graph is a strongly regular graph with parameters $(n,k,\lambda,\mu)$ if it has n nodes with all degree k, and any two adjacent vertices have $\lambda$ common neighbours, and any two non-adjacent vertices have $\mu$ common neighbours. A non-trivial (non-complete) strongly regular graph has diameter at most two, and are precisely the diameter-2 distance-regular graphs with intersection array parameters $b_0=k$, $a_1=\lambda$ and $c_2=\mu$.

\section{Discussions on \textsc{Graphtester} Dataset Results}
\label{sec:additional_results}

\subsection{Why Are Strongly Regular Graphs Hard?}
Strongly regular and distance-regular graphs, by definition, have certain properties of regularity that makes it difficult to distinguish nodes based on commonly used centrality measures. It turns out, this regularity is very well-connected with some of these metrics to the point that the parametrization of strongly regular graphs actually determine the values of some of these centrality metrics.

It is straightforward to note that for degree centrality $d(v)$, variance of the node degrees in a distance-regular graph is zero, that is, $var(d_G) = 0$. In 2012, Gago et al. have shown that this is the case for betweenness centrality $bc(v)$ as well, $var(bc_G) = 0$~\cite{2012-GagoBetwRegular}. One year later, Rombach et al. have proven the same for closeness centrality $cc(v)$ by noting that $var(cc_G) = 0$~\cite{2013-RombachClosenessRegular}. Below, we go one step further by using their results in the context of color refinement.

\begin{theorem}
    Let $G_1$ and $G_2$ be two strongly regular graphs with the same intersection array. Then, $d(v)= d(u)$ , $bc(v)= bc(u)$ , and $cc(v)=cc(u)$ for all $v \in V_{G_1}, u \in V_{G_2}$.
\end{theorem}
\begin{proof}
    Consider a strongly regular graph with parametrization (v,k,$\lambda$,$\mu$), i.e., with the following abbreviated intersection array:
    \begin{align*}
        \{b_0:k,b_1:k-1-\lambda;c_1:1,c_2:\mu\}.
    \end{align*}
    Then, degree centrality of all vertices in this graph is $k$ by definition.

    Using the results from Gago et al.~\cite{2012-GagoBetwRegular}, the geodesic ratio of a vertex v in this strongly regular graph can be written explicitly. Note that diameter of any strongly regular graph is 2, therefore $1 \leq d(u,v) \leq 2\: \forall u\neq v \in V$. Using the betweenness formula from the same work, we can expand the expression as follows (for some other nodes u,v of which w is on the shortest path in between, $d(a,b)$ is the distance between nodes $a$ and $b$)
    \begin{align*}
        bc(w) &= \sum_{u,v \neq w} \left(\dfrac{\prod_{i=1}^{d(u,w)} c_i}{\prod_{i=d(u,v)-d(u,w)}^{d(u,v)} c_i}\right)\\
        &= \sum_{u,v \neq w} \left(\dfrac{c_1}{\prod_{i=1}^{2} c_i}\right) \tag{$d(u,v)=2$ and $d(u,w)=d(v,w)=1$ as w is on the shortest path $u-v$}\\
        &= \sum_{u,v \neq w} \dfrac{1}{\mu}\\
        &= \dfrac{k(k-\lambda-1)}{2\mu}\tag{\#neighbors $\times$ \#non-adjacent neighbors, halved for repeated paths}
    \end{align*}
    Consequently, the value can be derived solely through the intersection array.
    
    For closeness centrality, we write the derivation of Rombach et al.~\cite{2013-RombachClosenessRegular} for distance-regular graphs, using the strongly regular graph parametrization above. We first define a generalization of the neighborhood of a node, $\Gamma^r(i) := \{j|d(i, j) = r\}$, as the nodes at distance $r$ from the source node $i$. Then, using the derivations from the same work, we can write closeness centrality of some node $w \in V$ as 
    \begin{align*}
        cc(w) &= \dfrac{n-1}{\sum_{r=1}^d r|\Gamma^r(w)|}\\
        &= \dfrac{n-1}{k + \dfrac{2k(k-\lambda-1)}{\mu}}\\
        &= \dfrac{\mu (n-1)}{\mu k + 2k(k-\lambda-1)}
    \end{align*}
    which only depends on the strongly regular graph parametrization.
\end{proof}
\begin{corollary}
    Strongly regular graphs of same parametrization are indistinguishable with degree, betweenness or closeness centralities.
\end{corollary}
\begin{corollary}
    For strongly regular graphs, 
    \begin{align*}
        cc(w) &= \dfrac{n-1}{d(w) + 4bc(w)}.
    \end{align*}
\end{corollary}

\subsection{A New Class of Features: Node Signature}

Resolving this inherent regularity is most likely not possibly with distance- or path-based centrality metrics. What are the components of a strongly regular graph that are not necessarily regular? It turns out, such a feature requires us to look at the subgraphs induced by the neighborhoods of each node.

\textit{$n$-th subconstituent} of a vertex, denoted by $\Gamma^n(v)$, is the subgraph induced by the nodes at distance $n$ from the node, minus itself. 1st subconstituent of a node is also called  \textit{local graph}. For strongly regular graphs, the second subconstituent at x is the graph induced on the set of vertices other than x and nonadjacent to x.

Turns out, such a graph is most of the times not strongly regular, and quite descriptive of the strongly regular graphs themselves~\cite{2022-BrouwerStronglyReg}. Using the observation from Table \ref{tab:gt_dataset_results} on the power of edge betweenness in distinguishing non-distance-regular graphs, we propose a new family of features named \textbf{$n$-th subconstituent signature} to use as pre-coloring for 1-WL algorithm, created with Algorithm \ref{alg:signature}.

\begin{algorithm}
\caption{$n$-th subconstituent signature}\label{alg:signature}
\begin{algorithmic}
\REQUIRE graph $G=(V,E)$, integer $n \geq 1$
\FOR {$v \gets V$}
    \STATE $E_v \gets EdgeBetweenness(\Gamma^n(v))$ \COMMENT{Outputs set of EB values per edge}
    \STATE $c_v \gets hash(E_v)$ \COMMENT{Permutation-invariant hash, e.g., sort and concatenate}
\ENDFOR
\STATE \textbf{return} $\{c_1, c_2,...,c_{|V|}\}$
\end{algorithmic}
\end{algorithm}

\begin{lemma}
    $n$-th subconstituent signatures of a connected graph can be estimated in worst-case complexity $\mathcal{O}(|V|m^3)$, where $m$ is the maximum degree.
\end{lemma}
\begin{proof}
    Using Brandes' algorithm, we can estimate the edge betweenness of undirected graphs in $\mathcal{O}(|V||E|)$~\cite{2004-BrandesEB}. For each node, the worst case estimation of the EB values for maximum node degree takes $\mathcal{O}(m^3)$ where $m$ is the maximum degree. Hashing and subgraph extraction can be done in $\mathcal{O}(m^2)$ and $\mathcal{O}(|V|+|E|)$ times respectively independent of $n$, latter done by running a breadth-first search.

    Finally, with the note that $m^3\geq m^2\geq |V|$ and $m^3\geq m^2\geq |E|$ for connected graphs, the worst-case complexity of Algorithm \ref{alg:signature} can be written as $\mathcal{O}(|V|m^3)$.
\end{proof}

\subsection{Benchmarking Results for the Signature}

Using the same \textsc{Graphtester} infrastructure to compare the power of signature with other methods, we created the Table \ref{tab:signature_results} that depicts the performance of the proposed method over alternatives on particularly challenging graph classes.

\begin{table}[!htb]
\centering
\caption{Failure counts of all tests conducted for all graph pairs in the given graph classes with noted orders for node signature features. Only noteworthy graph classes and methods are listed.}
\label{tab:signature_results}
\resizebox{\textwidth}{!}{%
\begin{tabular}{@{}c|cccccc|cccccc|ccccccc@{}}
\toprule
\multicolumn{1}{r|}{\textit{graph class:}} & \multicolumn{6}{c|}{All nonisomorphic} & \multicolumn{6}{c|}{Highly irregular} & \multicolumn{7}{c}{Strongly regular} \\ \cmidrule(l){2-20} 
\multicolumn{1}{r|}{\textit{node count:}}                     & 3 & 4 & 5 & 6 & 7  & 8   & 8 & 9 & 10 & 11 & 12  & 13 & 16 & 25  & 26 & 28 & 29  & 36    & 40  \\ \midrule
1-WL ($\equiv$ 2-WL)                                                            & 0 & 0 & 0 & 4 & 22 & 350 & 1 & 0 & 8  & 0  & 165 & 0  & 1  & 105 & 45 & 6  & 820 & 16110 & 378 \\
3-WL                                                          & 0 & 0 & 0 & 0 & 0  & 0   & 0 & 0 & 0  & 0  & 0   & 0  & 1  & 105 & 45 & 6  & 820 & 16110 & 378 \\
4-WL                                                          & 0 & 0 & 0 & 0 & 0  & 0   & 0 & 0 & 0  & 0  & 0   & 0  & 0  & 0   & 0  & 0  & 0   & 0     & 0   \\ \midrule
Neighborhood 1st subconstituent signatures                    & 0 & 0 & 0 & 0 & 0  & 6   & 0 & 0 & 1  & 0  & 3   & 0  & 0  & 0   & 0  & 0  & 0   & 1     & 2   \\
Neighborhood 2nd subconstituent signatures                    & 0 & 0 & 0 & 1 & 2  & 12  & 0 & 0 & 0  & 0  & 0   & 0  & 0  & 0   & 0  & 0  & 0   & 0     & 0   \\
\makecell{Neighborhood 1st subconstituent signatures +\\Edge betweenness} & 0 & 0 & 0 & 0 & 0  & 0   & 0 & 0 & 0  & 0  & 0   & 0  & 0  & 0   & 0  & 0  & 0   & 1     & 2   \\
\makecell{Neighborhood 2nd subconstituent signatures +\\Edge betweenness} & 0 & 0 & 0 & 0 & 0  & 0   & 0 & 0 & 0  & 0  & 0   & 0  & 0  & 0   & 0  & 0  & 0   & 0     & 0   \\ \bottomrule
\end{tabular}
}
\end{table}

This method, as well as other associated methods are provided as part of the \textsc{Graphtester} package. Note that this feature is not necessarily feasible as a positional encoding since it varies in dimension per node, therefore is required to be preprocessed by a dynamic and permutation-invariant layer.

\end{document}